\documentclass{article}

\usepackage{iclr2027_conference,times}
\usepackage{microtype}
\usepackage{graphicx}
\usepackage{booktabs}
\usepackage{hyperref}
\usepackage{amsmath}
\usepackage{amssymb}
\usepackage{amsthm}
\usepackage{url}
\usepackage{float}

\hypersetup{
  pdfsubject={Evidence provenance, grouping, and evidence weight in language models},
  hypertexnames=false,
  hidelinks
}

\newtheorem{theorem}{Theorem}
\newtheorem{proposition}[theorem]{Proposition}
\newtheorem{corollary}[theorem]{Corollary}

\newcommand{\known}{\textsc{Known}}

\newcommand{\PaperUniqueCheckpoints}{6}

\newcommand{\PhaseElevenChoiceTrials}{13,056}
\newcommand{\PhaseElevenCandidateForwardCalls}{26,112}

\newcommand{\PhaseElevenPositivePointCells}{16/16}

\newcommand{\PhaseTwelveChoiceTrials}{3,072}
\newcommand{\PhaseTwelveCandidateForwardCalls}{6,144}
\newcommand{\FinalPaperTrialsRevisionThree}{104,402}

\title{Record Grouping Controls Evidence Weight in Language Models}

\author{Zhongxuan Liu \qquad Sicheng Zhou \qquad Hongzhi Wang\thanks{Corresponding author.}\\
Faculty of Computing, Harbin Institute of Technology\\
\texttt{lazrix@163.com} \quad \texttt{ylnfq\_2021@qq.com}\\
\texttt{wangzh@hit.edu.cn}}

\iclrfinalcopy

\begin{document}
\maketitle
\lhead{Preprint}

\begin{abstract}
  Retrieved records are presentation units; a supplied partition determines which records enter a language model as one evidential contribution.  We characterize the invariant group--content state that removes within-group copies while retaining complementary canonical content, show that equal group counts can encode different evidence states, and derive a sharp content-aware partition-error bound.  Given a supplied partition, our pre-generation representation deduplicates and aggregates content within groups and bounds each group's contribution.  Across \FinalPaperTrialsRevisionThree{} trials and \PaperUniqueCheckpoints{} public checkpoints, a central natural-text intervention finds that content-fixed false splits add 10.27--32.66 percentage points and false merges remove 9.13--31.79 points; a matched six-slot control retains the positive direction in all 16 cells.  In a new 48-item controlled campaign panel, changing the supplied partition produces measurable, checkpoint-dependent decision shifts across all four models, and the balanced mirror design exposes substantial order interactions.  Together, the theory and experiments establish the supplied partition as a controllable pre-generation representation variable and characterize its checkpoint-dependent behavioral effects.
\end{abstract}

\section{Introduction}

Retrieval-augmented systems consume packaging records whose multiplicity can diverge from the number of underlying evidence sources.  A coordinated campaign can render one information-generating process across many pages; raw retrieval serialization then exposes each rendering as a separate model-facing contribution opportunity.  This matters for answer engines, search agents, product recommendation, and any system exposed to coordinated content placement because every choice of chunking, syndication handling, and record emission implicitly chooses an evidence-weighting rule.

Prior work has established that frequency, repeated arguments, source labels, metadata, and document diversity can steer language-model decisions \citep{jin-etal-2024-tug,wan-etal-2024-evidence,chiang-lee-2024-metadata,schuster2026whose,naphade-2026-rational,ross2026redundancy}.  We take this behavioral susceptibility as the starting point.  The systems question is: \emph{which state preserves the content of each supplied evidence unit while removing within-unit copies?}  Once an upstream partition specifies which records share a dependence unit, the representation should preserve complementary content while allocating one bounded contribution per group.  This guarantee resides in the model-facing state and holds independently of behavioral tendencies, prompt instructions, and trained preferences.

Group count alone leaves an essential choice unresolved.  Four distinct evidence elements can be grouped as $\{\{1,2\},\{3,4\}\}$ or $\{\{1,3\},\{2,4\}\}$: both have two groups of size two, yet combine different facts within each unit.  We prove that even bounded group aggregation can distinguish these states.  This motivates retaining group membership together with content.

We group records before generation.  An upstream process supplies the partition.  Model-facing membership labels are opaque: equality encodes membership, while authenticated identity or credibility metadata, when available, live in explicit group signatures.  Within each group, the representation removes exact copies, aggregates complementary passages, and assigns one bounded contribution.  Figure~\ref{fig:introduction} summarizes the introductory logic from repeated records and shared roots to the supplied partition and its two principal error modes.

Table~\ref{tab:phase9-main} provides the central causal test.  Its full-content contrast adds complementary passages within one supplied group; its split and merge contrasts hold the ordered 160 evidence words fixed while changing the emitted partition.  The matched six-slot control fixes every repeated non-content field and tokenizer length.  The experiments condition on supplied operational grouping keys and isolate the downstream record partition.

\begin{figure}[t]
    \centering
    \includegraphics[width=\linewidth]{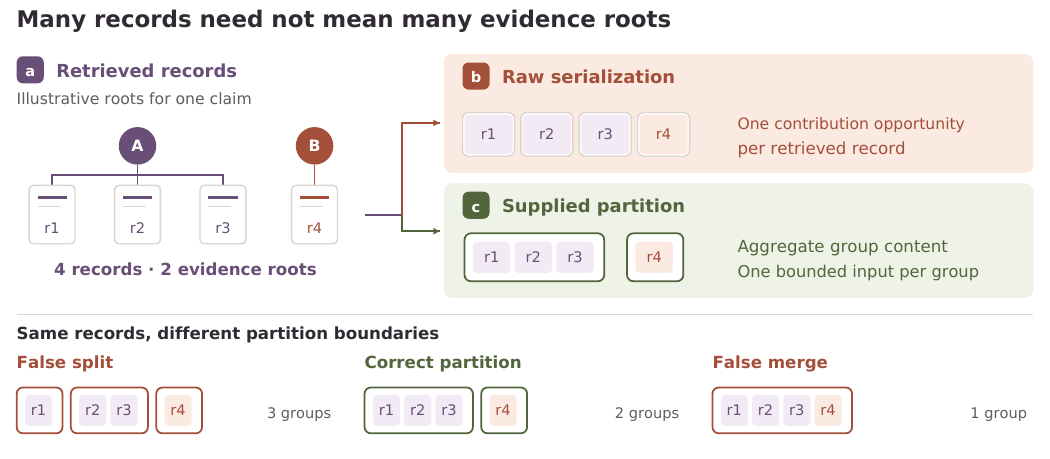}
    \caption{Retrieved records are presentation units and may share claim-relative evidence roots.  Raw serialization allocates weight record by record; a supplied partition enables within-group aggregation and one bounded contribution per group.  False splits and false merges are the two central partition errors.}
    \label{fig:introduction}
\end{figure}

Our contributions are threefold.
\begin{itemize}
    \item We give a pre-generation grouping representation that combines a supplied partition, within-group content aggregation, and one bounded contribution per group.  Exact copies map to the same model-facing state; explicit signatures carry authenticated identity metadata.
    \item We characterize invariant and faithful group--content representations, establish equal-count separation, and derive a sharp content-aware error bound.  Claim-relative roots and the exact text-only recovery boundary specify the upstream provenance problem.
    \item We provide a three-contrast causal decomposition and a matched six-slot control that fixes object count, headers, identifiers, separators, and tokenizer length.  A new controlled panel extends the intervention from exact or partial copies to four cross-genre renderings of one campaign root and a four-independent-root control, revealing checkpoint-specific reversals.
\end{itemize}

Figure~\ref{fig:overview} connects the representation to the theory, controlled interventions, and empirical findings developed in the remainder of the paper.

\begin{figure}[t]
    \centering
    \includegraphics[width=\linewidth]{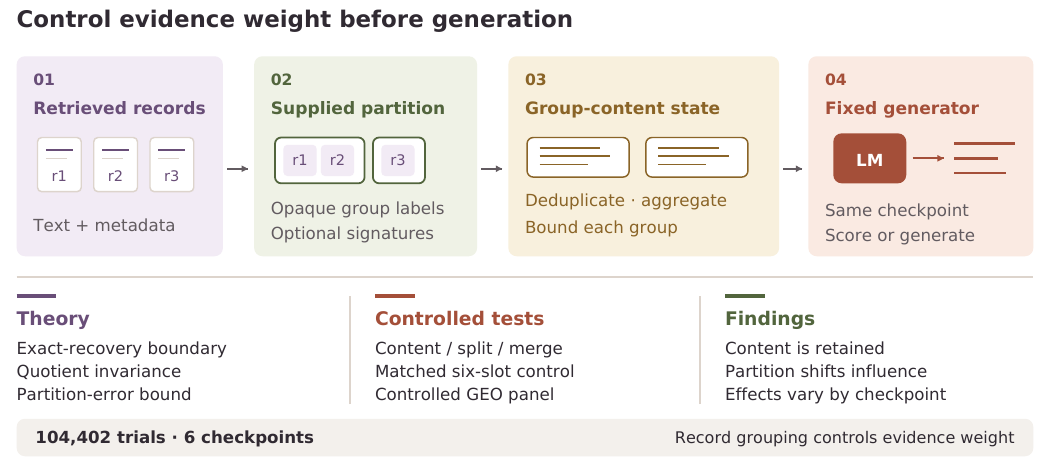}
    \caption{From retrieved records to model-facing influence.  A supplied claim-relative partition maps retrieved records into a group--content state before a fixed generator.  The paper characterizes this representation theoretically, tests it with content-fixed and matched controls, and measures how record boundaries reallocate model-facing influence across checkpoints.}
    \label{fig:overview}
\end{figure}

\section{Related Work}

\paragraph{Context conflicts and source preferences.}
Language models arbitrate inconsistently between parametric memory, contextual claims, and user-supplied specifications \citep{longpre-etal-2021-entity,zhou-etal-2023-context,zhou2024establishing}.  Evidence frequency, rationale framing, metadata, institutional labels, and source authority further alter this balance \citep{jin-etal-2024-tug,sun-etal-2026-task,wan-etal-2024-evidence,chiang-lee-2024-metadata,ge2025resolving,li-etal-2025-llms-trust,liao2026auditing}.  These studies establish that contextual frequency and source cues are behaviorally active.  Our question begins at the next layer: how an external representation makes replication invariance a property of the model-facing state with enforcement independent of checkpoint responses.

\paragraph{Redundancy, diversity, and grouped evidence.}
GroupQA shows that paraphrased documents supporting one argument can outweigh distinct support and documents order effects and unfaithful explanations \citep{naphade-2026-rational}.  Whose Facts Win? shows that repetition can reverse credibility preferences across 13 open-weight models and combines teacher--student LoRA distillation with a credibility-aware prompt to induce approximate repetition invariance \citep{schuster2026whose}.  In a benign fictional-QA setting, \citet{ross2026redundancy} find little correctness gain from duplicates or paraphrases and a large gain from diverse documents.  These works define the closest behavioral and learned-mitigation frontier.  \citet{song2026aggregation} attach canonical identities to observations and maintain mergeable aggregation states outside long-context generation.  We study the representation itself: supplied source-dependence keys induce a group--content state before generation, and content-fixed split/merge interventions with matched six-slot serialization identify the effect of record partition on evidence weight.

\paragraph{Adversarial retrieval and source grouping.}
SearchGEO and generative-engine optimization supply the adversarial application: coordinated web publication can manipulate the evidence retrieved and endorsed by answer systems \citep{chen2026searchgeo,aggarwal2024geo}.  \citet{rahadi2026counting} proposes provenance-graph estimation, effective independent evidence count and confidence-inflation diagnostics, and copy-cluster-discounted answer aggregation.  Its aggregation uses membership information as well as counts.  Our complementary question concerns the state supplied before generation: which invariances it enforces, which within-group content it preserves, and how partition error distorts it.  Equal-count separation below concerns scalar count summaries; the controlled interventions establish checkpoint-dependent responses to record boundaries.  Adversarial tool outputs expose the same problem for agents \citep{zhan-etal-2026-adversarial}, while instruction-hierarchy training improves prioritization of privileged over untrusted inputs \citep{wallace2024instruction}.  The characterization uses standard quotient factorization, with the evidence-specific content and error semantics made explicit.

\section{Partitioned Evidence Before Generation}

\subsection{Claim-Relative Evidence Roots}

For a target claim $c$ and fixed observed record set $I$, let $\rho_c(r)\in\mathcal S_c$ denote the information-generating root from which record $r$ derives its claim-relevant evidence.  Each $r\in I$ is a claim-evidence atom with exactly one claim-relative root.  A physical document drawing on multiple roots enters $I$ only after atomization into such units or assignment to an explicitly defined composite root.  Here ``independent'' denotes distinct roots under this provenance relation; statistical dependence among their observations remains admissible.  The corresponding oracle partition is
\begin{equation}
P_c^\star=\bigl\{\{r\in I:\rho_c(r)=s\}:s\in\rho_c(I)\bigr\}.
\label{eq:evidence-root-partition}
\end{equation}
When admissible worlds vary below, $P_c^\star(w)$ denotes the partition obtained in world $w$; elsewhere we suppress $w$.  Its block count is the \emph{epistemic multiplicity} of the retrieved evidence for $c$; the number of returned records is its presentation multiplicity.  The distinction is claim-relative: two articles can share a measurement for one claim while carrying independently produced evidence for another.

An upstream dependence-discovery process estimates $P_c^\star$ as $\widehat P_c$.  The downstream mechanism accepts either partition as an input.  Opaque labels encode membership; authenticated signatures carry provenance and credibility.  Under $P_c^\star$, another rendering from an existing root may enrich canonical content inside that block while the group-level summand count stays fixed.

For a fixed claim and record count, let $Z_c(w)$ be the complete visible text in an admissible world $w$.  Universal deterministic text-only recovery exists exactly when $Z_c(w_0)=Z_c(w_1)$ implies $P_c^\star(w_0)=P_c^\star(w_1)$ (Proposition~\ref{prop:text-only-provenance}).

Equivalently, the oracle partition is constant on every observational fiber of $Z_c$.  Similarity models provide useful evidence under distributional assumptions; universal exact recovery is available precisely when worlds with identical visible text share the same oracle partition.  Appendix~\ref{app:text-only-proof} proves the characterization.

\subsection{Invariant and Faithful Group--Content State}

The representation retains the canonical content of every supplied group while removing repeated occurrences within that group.

Let $\mathcal K$ be an infinite set of opaque membership labels and let $\mathcal Z_{\rm ev}$ contain canonical evidence elements.  A label's spelling carries no evidential meaning; equality and inequality encode group membership.  An upstream process supplies the partition, and authenticated provenance, when available, can justify it; the model-facing label remains local bookkeeping.  A finite grouped configuration is $x=((k_i,z_i))_{i=1}^n\in\mathsf{Cfg}=(\mathcal K\times\mathcal Z_{\rm ev})^*$.  Let $K_x$ be its set of occupied labels and, for each $k\in K_x$, let $V_x(k)=\{z_i:k_i=k\}$.  The canonical evidence quotient is the finite multiset
\begin{equation}
q_{\rm ev}(x)=\sum_{k\in K_x}\delta_{V_x(k)}\in
\mathsf{Quot}=\mathsf M_{\rm fin}(\operatorname{Fin}^+(\mathcal Z_{\rm ev})).
\label{eq:ceq-map}
\end{equation}
The outer object is a multiset, so distinct groups with identical content remain distinct; each inner object is a set, so another canonical copy inside one group disappears.  Declare $x\sim y$ when $q_{\rm ev}(x)=q_{\rm ev}(y)$.  This relation removes record order, bijective renaming of opaque occupied labels, and within-group canonical copies while preserving complementary elements and every nontrivial split or merge as distinct quotient states.  Appendix~\ref{app:ceq-proofs} gives the signature-augmented version that carries authenticated identity or credibility as an evidence-bearing coordinate.

The representation has three nuisance invariances: record permutation, bijective renaming of occupied opaque labels, and within-group insertion or deletion of an exact canonical duplicate while one occurrence remains.  Faithfulness asks that all distinctions between the resulting group--content states remain recoverable.
\begin{theorem}[Invariant and faithful evidence representations]
\label{prop:ceq-universal-factorization}
For any set $Y$ and representation $\Psi:\mathsf{Cfg}\to Y$, the following are equivalent: (i) $\Psi$ has the three nuisance invariances above; (ii) $\Psi(x)=\Psi(y)$ whenever $x\sim y$; and (iii) there is a unique $\bar\Psi:\mathsf{Quot}\to Y$ such that
\begin{equation}
\Psi=\bar\Psi\circ q_{\rm ev}.
\label{eq:ceq-factorization}
\end{equation}
Among these invariant representations, preserving every relation-invariant observable is equivalent to injectivity of $\bar\Psi$, and to recoverability of $q_{\rm ev}$ from $\Psi$.  Thus faithful invariant states are precisely injective recodings of $q_{\rm ev}$.
\end{theorem}

The three operations generate exactly the fibers of $q_{\rm ev}$: match groups with equal content sets, rename their labels, delete duplicate occurrences, and permute records.  Standard quotient factorization then gives the characterization; choosing $q_{\rm ev}$ itself as the observable gives recoverability.  The minimality is relative to preserving \emph{all} nuisance-invariant information under a fixed canonicalizer.  A specified downstream task can use a coarser state.  Appendix~\ref{app:ceq-proofs} supplies the full proofs and the downstream-map formulation.

The group-additive construction is one factorized realization.  Let $U$ be an ambient record universe, let $z_{\rm ev}:U\to\mathcal Z_{\rm ev}$ be one fixed canonicalizer, let $(\mathcal X,\|\cdot\|)$ be a real normed vector space, and let $a$ map finite subsets of $\mathcal Z_{\rm ev}$ into $\mathcal X$.  For a finite $I\subseteq U$ and a partition $P$ of $I$, define
\begin{equation}
    E_I(P)=\sum_{G\in P}a\!\left(\{z_{\rm ev}(i):i\in G\}\right).
    \label{eq:grouped-evidence}
\end{equation}
Writing $q_P=\sum_{G\in P}\delta_{\{z_{\rm ev}(i):i\in G\}}$ and $\Phi_a(m)=\sum_Sm(S)a(S)$ gives $E_I(P)=\Phi_a(q_P)$.  Hence $E_I$ always factors through the quotient; quotient faithfulness holds exactly when the aggregator separates the reachable quotient states, as characterized in Appendix~\ref{app:ceq-proofs}.

\begin{proposition}[Equal-count separation]
\label{prop:equal-count-separation}
Let $z_1,z_2,z_3,z_4$ be distinct canonical elements attached to four fixed records, and set
\[
P_A=\{\{1,2\},\{3,4\}\},\qquad
P_B=\{\{1,3\},\{2,4\}\}.
\]
The record content, group count, and block-size multiset agree, but $q_{P_A}\ne q_{P_B}$.  For every $B>0$ there is a fixed scalar aggregator with $|a(S)|\le B$ for which $E_I(P_A)=2B$ and $E_I(P_B)=-2B$.
\end{proposition}

For the witness, assign $+B$ to $\{z_1,z_2\}$ and $\{z_3,z_4\}$, $-B$ to $\{z_1,z_3\}$ and $\{z_2,z_4\}$, and zero elsewhere.  Thus the same global content and count can support distinct bounded evidence states.  Count records how many units exist; the group--content state also records which facts they combine.  This is a representation-level separation; the model experiments below measure content-fixed split/merge responses.

\begin{proposition}[Exact-replication invariance]
\label{prop:external-replication-invariance}
If a new record $i'\in U\setminus I$ is added to the block containing $i$ and $z_{\rm ev}(i')=z_{\rm ev}(i)$, then the evidence state in Equation~\eqref{eq:grouped-evidence} is unchanged.
\end{proposition}

The construction uses a supplied grouping key, a fixed canonicalizer, and an extensional set-valued group input; its proof and the corresponding grouping-error bound are in Appendix~\ref{app:grouping-guarantees}.  A canonicalizer determines which copies share an element, while complementary passages remain separate elements within one group.  More generally, let $J\subseteq U\setminus I$ be a finite set of new records that joins an existing block $G\in P$, and set $P'=(P\setminus\{G\})\cup\{G\cup J\}$.  Their entire effect is
\begin{equation}
E_{I\cup J}(P')-E_I(P)
=a(A_G\cup A_J)-a(A_G),
\label{eq:same-root-proliferation}
\end{equation}
where $A_G=\{z_{\rm ev}(i):i\in G\}$ and $A_J=\{z_{\rm ev}(j):j\in J\}$.  When $P=P_c^\star$ and every record in $J$ shares the claim-relative root of $G$, this is same-root proliferation.  The number of outer contributions stays fixed; the state changes only through genuinely new canonical content inside the root.

\subsection{From Partition Error to Decision Stability}
For two partitions $P,Q$ of the same records, define the content-aware discrepancy
\[
\Delta_q(P,Q)=\|q_P-q_Q\|_1
=\sum_A |q_P(A)-q_Q(A)|.
\]
It cancels matching group--content sets even when their record identities differ.  With a fixed canonicalizer and $\|a(A)\|\le B$, Proposition~\ref{prop:content-aware-error} gives
\begin{equation}
\|E_I(P)-E_I(Q)\|\le B\Delta_q(P,Q)\le BH(P,Q),
\label{eq:provenance-discovery-bound}
\end{equation}
where $H$ counts the blocks in changed overlap components (Appendix~\ref{app:grouping-guarantees}).  The first bound is exact in the worst case over bounded scalar aggregators for each fixed pair $P,Q$.  Equal-count reassignment can attain $4B$, as the preceding witness shows.  For an $L$-Lipschitz decision margin $m$, its sign is stable whenever $|m(E_I(P))|>LB\Delta_q(P,Q)$.  Taking $P=P_c^\star$ and $Q=\widehat P_c$ connects upstream partition error to representation distortion and then to decision stability.  The numerical constants belong to a specified representation and downstream map; the language-model experiments measure behavioral effects separately.

The natural-text experiment uses externally supplied grouping metadata.  HUMAN keys denote crowd-annotated evidence units; WEB copies share canonical URLs, and its Four-Unit records satisfy frozen domain, URL, and text-hash separation criteria.  These are operational units, while authenticated identity and credibility remain explicit signature coordinates.

\section{Experimental Program}

\subsection{Panels, Scoring, and Statistics}

The natural audit analyzes 101 PERSPECTRUM claims in 49 dependence components \citep{chen2019perspectrum} and 138 ConflictingQA questions in 135 components \citep{wan-etal-2024-evidence,chiang-lee-2024-metadata} as separate domains.  The grouping panel uses 66 HUMAN claims passing a frozen document filter and all 138 WEB questions; generation uses 40 components per domain.

Candidate-choice audits score two complete, length-matched assistant answers with native termination sequences by total conditional log likelihood.  Exact ties score zero in the primary analysis and one half in sensitivity analysis.  Controlled generation separately scores the parsed leading \texttt{Yes}/\texttt{No} over all outputs.

For trial $i$, let $x_i$ be the rendered prompt, let $a_i$ and $b_i$ be the attack-side and opposite complete candidates, and let $u_{i,c,1:m_{ic}}$ be candidate $c$'s scoring suffix, including the checkpoint's native termination sequence.  We compute
\begin{equation}
\begin{aligned}
\ell_i(c)
&=\sum_{t=1}^{m_{ic}}\log p_\theta(u_{i,c,t}\mid x_i,u_{i,c,<t}),\\
p_{{\rm attack},i}
&=\frac{\exp\ell_i(a_i)}{\exp\ell_i(a_i)+\exp\ell_i(b_i)}
=\frac{1}{1+\exp(\ell_i(b_i)-\ell_i(a_i))}.
\end{aligned}
\label{eq:attack-probability}
\end{equation}
Thus $p_{\rm attack}$ is the likelihood share over the two frozen candidates.

We average mirrored assignments and orders within item before each contrast.  Fixed-seed 10,000-replicate bootstraps resample items or natural-text dependence components; estimates remain checkpoint- and domain-specific.  Appendix~\ref{app:phase8-specification}--\ref{app:extended-specifications} gives full specifications.

\subsection{Central Grouping Intervention and Breadth Studies}

For the one-group partial-copy panel, four non-overlapping 40-word windows come from one frozen document.  For the four-group panel, four windows come from four operationally distinct supplied units.  We hold the supplied keys fixed, inject false splits or false merges downstream, aggregate content within resulting groups, and emit one evidence record per estimated group.  The main comparisons hold the ordered 160 attack words fixed.  Across Qwen3-8B and Qwen3-4B \citep{yang2025qwen3}, Phi-4-mini \citep{microsoft2025phi4mini}, and Mistral-7B \citep{jiang2023mistral}, this experiment contributes 39,168 trials.  A further \PhaseElevenChoiceTrials{}-trial control uses six JSON slots in both arms, retains R1--R6 and every header and separator, matches final character, byte, and tokenizer lengths, and varies whether the fixed word stream occupies one or four nonempty content-bearing slots.

\paragraph{Controlled campaign and independent-root panel.}
We construct 48 fictional product pairs.  For each attack side, four 40-word cross-genre records deterministically restate one frozen three-fact campaign brief and share one oracle root; a second family contains four 40-word records tied to independently specified laboratory, panel, endurance, and service-data roots.  Every visible product claim is bound to an enumerated atomic fact before inference.  Hidden root, author, and synthetic-domain fields never enter the prompt.  In both families, the Phase 11 six-slot renderer compares $[160,0,0,0]$ with $[40,40,40,40]$ while fixing the ordered word stream, six objects, R1--R6 fields, side sequence, final characters, UTF-8 bytes, and model-specific token length.  Two attack sides and two order mirrors yield \PhaseTwelveChoiceTrials{} choice trials over the same four checkpoints.

Supporting audits cover authority cues, explicit relation/count stages, repetition dose, natural-text copies, added Qwen3-14B-AWQ and Phi-4 checkpoints \citep{yang2025qwen3,abdin2024phi4}, quantization, and controlled generation.  Together with the central grouping, matched-serialization, and provenance panels, the program contains \FinalPaperTrialsRevisionThree{} trials across \PaperUniqueCheckpoints{} unique checkpoints.  Appendix~\ref{app:repro} gives complete designs, counts, and lineage.

\section{Results}

\subsection{Grouping Separates Content Retention from Partition Effects}
\label{sec:phase9-results}

\begin{table}[t]
\caption{Central causal decomposition of content retention and record partition, in percentage points with 95\% dependence-component bootstrap intervals. Full-content gain compares the 160-word one-group record with its 40-word endpoint. The split and merge contrasts hold the ordered 160 attack words fixed; positive values denote influence gained by splitting one supplied group or influence lost by merging four supplied groups. The columns identify complementary-content retention, false-split inflation, and false-merge suppression. Cells report claims/dependence components: HUMAN 66/32 and WEB 138/135.}
\label{tab:phase9-main}
\centering
\footnotesize
\setlength{\tabcolsep}{1.3pt}
\renewcommand{\arraystretch}{1.08}
\begin{tabular}{@{}llrrr@{}}
\toprule
Checkpoint & Domain & \shortstack{Full-content\\gain} & \shortstack{One supplied\\unit: split} & \shortstack{Four supplied\\units: merge loss} \\
\midrule
Qwen3-8B & HUMAN (66/32) & $+13.98[+8.65,+19.82]$ & $+26.38[+20.51,+32.36]$ & $+25.17[+18.26,+33.11]$ \\
Qwen3-8B & WEB (138/135) & $+14.30[+11.24,+17.51]$ & $+25.23[+21.43,+29.19]$ & $+24.65[+20.55,+28.77]$ \\
Qwen3-4B & HUMAN (66/32) & $+29.82[+25.39,+34.45]$ & $+20.61[+15.38,+25.68]$ & $+31.79[+27.05,+38.06]$ \\
Qwen3-4B & WEB (138/135) & $+15.35[+11.80,+18.82]$ & $+24.43[+20.99,+28.01]$ & $+31.56[+27.46,+35.81]$ \\
Phi-4-mini & HUMAN (66/32) & $+4.96[+2.49,+8.14]$ & $+14.46[+9.81,+19.94]$ & $+9.13[+6.02,+13.08]$ \\
Phi-4-mini & WEB (138/135) & $+2.91[+1.07,+4.80]$ & $+10.27[+7.64,+13.23]$ & $+10.10[+7.65,+12.76]$ \\
Mistral-7B & HUMAN (66/32) & $+18.93[+14.27,+24.16]$ & $+32.66[+28.59,+37.74]$ & $+31.15[+27.18,+35.64]$ \\
Mistral-7B & WEB (138/135) & $+18.47[+14.99,+22.08]$ & $+25.21[+21.99,+28.49]$ & $+26.69[+23.00,+30.51]$ \\
\bottomrule
\end{tabular}
\end{table}

Table~\ref{tab:phase9-main} is the paper's central causal-identification test.  It holds the supplied operational keys fixed, intervenes on downstream aggregation and rendering, and decomposes evidence weighting into three interpretable contrasts.

First, the full-content arm exact-deduplicates and concatenates four non-overlapping windows from one supplied unit into a single 160-word group record.  Relative to the one-window endpoint, the additional complementary content raises attack-side probability by 2.91--29.82 points, with all eight confidence intervals above zero.  One-group normalization therefore combines complementary within-unit content retention with multiplicity control.

Second, falsely splitting that same supplied unit into four records raises attack-side probability by 10.27--32.66 points.  Third, falsely merging four supplied units into one record suppresses their influence by 9.13--31.79 points.  Every interval is positive.  In both contrasts, the evidence words and their order are identical across arms; after the supplied keys are fixed, the intervention changes the record partition and its serialization.  Together, the three contrasts disentangle complementary-content retention, multiplicity inflation from false splits, and influence suppression from false merges.

The original renderer fixes the instruction, query, claim, anchor content, chat adapter, candidates, and scoring suffix.  Its four-record implementation introduces three additional JSON objects with headers, separators, and ordinal identifiers, corresponding to 43--63 additional input tokens across the four tokenizers and an evidence-token difference bounded by five.  The matched six-slot control fixes the full serialization skeleton: both arms contain the same six objects, R1--R6 identifiers, side labels, headers, separators, template text, ordered evidence word stream, final characters and bytes, and tokenizer input length; the content spans one or four attack-side slots.  Its effects are positive in \PhaseElevenPositivePointCells{} cells and 12 confidence intervals lie entirely above zero, ranging from $+0.63$ to $+13.33$ points.  The matched design identifies content-bearing record placement under equal object, header, and token counts.  Outcome-blind tail whitespace maintains exact tokenizer length.  Appendix~\ref{app:extended-specifications} gives the full accounting and model-specific estimates.

A post-hoc multiplicity sensitivity applies Bonferroni familywise 95\% bootstrap intervals to the 24 Table~\ref{tab:phase9-main} cells.  Positive lower bounds remain for all 16 content-fixed split/merge effects and seven of eight content-retention effects; only Phi-4-mini/WEB content retention crosses zero by 0.02 points.  The matched six-slot family retains 11 of 16 positive lower bounds after the same correction (Appendix~\ref{app:multiplicity-sensitivity}).

\subsection{Controlled GEO Renderings Expose Checkpoint-Dependent Partition Effects}
\label{sec:phase12-results}

\begin{table}[t]
\centering
\caption{Controlled campaign and independent-root interventions.  Values are percentage-point changes in $p_{\rm attack}$ with pointwise 95\% dependence-component bootstrap intervals.  Same-root split is four content-bearing records minus one grouped record; independent-root merge loss is four records minus one falsely merged record.  Positive values are the preregistered directions; checkpoints are reported without pooling.}
\label{tab:phase12-geo}
\small
\begin{tabular}{lrr}
\toprule
Model & \shortstack{Same-root\\split} & \shortstack{Independent-root\\merge loss} \\
\midrule
Qwen3-8B & $+18.46[+15.55,+21.41]$ & $-4.25[-6.86,-1.87]$ \\
Qwen3-4B & $+9.66[+6.59,+12.54]$ & $+5.39[+3.12,+7.94]$ \\
Phi-4-mini & $+27.44[+24.64,+30.25]$ & $+15.73[+12.87,+18.70]$ \\
Mistral-7B & $-4.44[-7.35,-1.64]$ & $+3.02[+1.63,+4.61]$ \\
\bottomrule
\end{tabular}
\end{table}

Table~\ref{tab:phase12-geo} extends the content-fixed six-slot intervention to cross-genre records from one campaign root and to four independently specified roots.  The same-root split effect is positive for Qwen3-8B, Qwen3-4B, and Phi-4-mini but negative for Mistral-7B.  Thus coordinated proliferation can gain model-facing influence without exact copying, while the Mistral reversal establishes a checkpoint-dependent behavioral sign.

The independent-root merge loss is positive for Qwen3-4B, Phi-4-mini, and Mistral-7B but reverses for Qwen3-8B.  The balanced average is a boundary-sensitivity stress test: mirror decomposition shows large order interactions (for Qwen3-8B, the independent-root contrast is $+41.43$ points in original order and $-49.92$ in reverse order).  Record-boundary placement is therefore behaviorally active, with its sign determined jointly by checkpoint, content family, and presentation.  The representation guarantee fixes one bounded group-level contribution for each supplied group; checkpoint-specific behavioral directions remain empirical.

Appendix Table~\ref{tab:geo-order-decomposition} reports a paired post-hoc order decomposition of these same predictions for every checkpoint and root family.  It estimates the order interaction within each item before resampling, preserving the two attack-side mirrors.  The balanced effects in Table~\ref{tab:phase12-geo} are recovered exactly by averaging the two orders.

\subsection{Record Copies Reweight Candidate and Generated Decisions}
\label{sec:phase8-results}

Across six checkpoints, four copies of one supplied unit shift attack-side candidate decisions by 9.24--51.98 points; every model--domain interval excludes replication invariance in the observed direction (Figure~\ref{fig:phase8} and Appendix Tables~\ref{tab:phase8} and~\ref{tab:phase9-breadth}).

The original HUMAN and WEB ranges are 14.11--49.01 and 9.24--37.14 points.  Added 14B checkpoints reach 30.62--51.98 points.  In controlled answer-plus-one-sentence generation, raw copying moves the leading answer by 15.62--20.62 points across three checkpoints, with positive lower bounds and 2,880/2,880 valid outputs.  Appendix~\ref{app:extended-specifications} reports prompt-rule heterogeneity, threshold sensitivity, and 956 byte-identical grouping-recovery checks per compiled grid.

\section{Discussion}

\paragraph{Partition is an evidence-accounting decision.}
The same claim can reach a generator as one record, several copied records, several complementary passages, or several records carrying one supplied key.  Table~\ref{tab:phase9-main} shows that this packaging actively allocates evidential influence: complementary content survives one-group aggregation, false splits inflate influence, and false merges suppress it.  The matched six-slot control preserves the predicted direction in every cell after equalizing serialization counts and tokenizer length.  Table~\ref{tab:phase12-geo} shows that cross-genre boundary placement remains active while its sign can reverse across checkpoints.  A RAG pipeline therefore chooses an evidence-weighting rule whenever it chooses how to chunk, duplicate, group, and serialize retrieval results.  Evidence independence tracks claim-relevant information-generating roots; surface diversity and URL count describe presentation, and fixed-model sensitivity to serialization is measured separately.

\paragraph{Dependence-aware weighting belongs before generation.}
A source-aware system separates four responsibilities: discover an upstream dependence partition, group records, aggregate complementary within-group content, and normalize group-level weight.  The supplied blocks can encode exact record identity, article lineage, or an application-defined claim-relative evidence root.  Model-facing membership labels remain opaque, while explicit group signatures carry authenticated identity and credibility.  On the controlled panel, exact hash and MinHash miss every cross-genre same-root grouping, whereas a fixed-threshold sentence-embedding baseline \citep{reimers-gurevych-2019-sentence} separates both provenance structures; Appendix Table~\ref{tab:phase12-grouping-baselines} reports the full partition audit.  The deliberately structured templates provide a controlled-panel sanity check for semantic grouping.  Proposition~\ref{prop:text-only-provenance} places universal exact recovery at observational-fiber constancy.  Once a partition is supplied, aggregation and per-group normalization enforce the chosen accounting rule independently of a checkpoint's behavioral sign.

\paragraph{What information must survive grouping.}
Theorem~\ref{prop:ceq-universal-factorization} distinguishes invariance from faithful content retention: an invariant summary preserves the full group--content state exactly when its recoding is injective.  Proposition~\ref{prop:equal-count-separation} exhibits the information discarded by a scalar group count even with identical records and equal block sizes.  The sharper discrepancy $\Delta_q$ makes the same distinction in error propagation, charging the unmatched group--content mass.  Together these results specify what the supplied partition controls before generation.  The observed sign reversals motivate checkpoint-specific measurement of the response to that state.

\section{Conclusion}

A supplied partition specifies which records contribute together before generation.  Its group--content state preserves distinctions beyond record count and group count, while removing within-group canonical copies.  Across exact copies, complementary passages, content-fixed split/merge interventions, and controlled campaign renderings, the emitted partition changes model-facing influence.  The campaign panel exhibits checkpoint-dependent sign reversals.  The formal guarantee belongs to the external representation: each supplied group receives one bounded group-level contribution while complementary canonical content remains available within that group.  Under oracle grouping, each group corresponds to one claim-relative root; text, metadata, and authenticated provenance signals inform the estimated partition, explicit signatures preserve evidence-bearing identity, and Equation~\eqref{eq:provenance-discovery-bound} converts partition errors into a representation bound.

\subsection*{AI use statement}
None.

\subsection*{Ethics statement}
The study analyzes public, previously released datasets, including prior crowd annotations in the HUMAN domain, and introduces zero new human-subject interactions.  All analyzed data are public.  The main misuse risk is content manipulation informed by the controlled repetition protocol.  The presentation centers defensive evaluation and separates oracle source metadata from editable text labels.

\subsection*{Reproducibility statement}
The anonymous artifact combines the verified stored-output layer with the controlled GEO extension and recomputes packaged Phase 6 and Phase 8--12 results.  Its Phase 12 layer contains all 3,072 choices, predictions, grouping assignments, five paper outputs, and the independent audit.  A clean extraction verifies 284 manifest items and recomputes legacy and Phase 12 statistics with zero failures; the maximum Phase 12 numerical difference is $5.33\times10^{-15}$.  Phase 5 and Phase 7 execution counts remain in their audited run records.  Fresh forward execution uses the named public checkpoints and licensed datasets.

The accompanying analysis scripts reproduce the additional post-hoc order decomposition from these stored Phase 12 predictions and check the preservation of the original result files.

\bibliography{references}

\begin{thebibliography}{25}
\providecommand{\natexlab}[1]{#1}
\providecommand{\url}[1]{\texttt{#1}}
\expandafter\ifx\csname urlstyle\endcsname\relax
  \providecommand{\doi}[1]{doi: #1}\else
  \providecommand{\doi}{doi: \begingroup \urlstyle{rm}\Url}\fi

\bibitem[Abdin et~al.(2024)]{abdin2024phi4}
Marah Abdin et~al.
\newblock {Phi-4} technical report.
\newblock \emph{arXiv preprint arXiv:2412.08905}, 2024.
\newblock URL \url{https://arxiv.org/abs/2412.08905}.

\bibitem[Aggarwal et~al.(2024)Aggarwal, Murahari, Rajpurohit, Kalyan,
  Narasimhan, and Deshpande]{aggarwal2024geo}
Pranjal Aggarwal, Vishvak Murahari, Tanmay Rajpurohit, Ashwin Kalyan, Karthik
  Narasimhan, and Ameet Deshpande.
\newblock {GEO}: Generative engine optimization.
\newblock In \emph{Proceedings of the 30th ACM SIGKDD Conference on Knowledge
  Discovery and Data Mining}, pp.\  5--16, 2024.
\newblock \doi{10.1145/3637528.3671900}.
\newblock URL \url{https://doi.org/10.1145/3637528.3671900}.

\bibitem[Chen et~al.(2019)Chen, Khashabi, Yin, Callison-Burch, and
  Roth]{chen2019perspectrum}
Sihao Chen, Daniel Khashabi, Wenpeng Yin, Chris Callison-Burch, and Dan Roth.
\newblock Seeing things from a different angle: Discovering diverse
  perspectives about claims.
\newblock In \emph{Proceedings of the 2019 Conference of the North American
  Chapter of the Association for Computational Linguistics: Human Language
  Technologies, Volume 1 (Long and Short Papers)}, pp.\  542--557, 2019.
\newblock \doi{10.18653/v1/N19-1053}.
\newblock URL \url{https://aclanthology.org/N19-1053/}.

\bibitem[Chen et~al.(2026)Chen, Ren, Laakom, Li, Guo, and
  Schmidhuber]{chen2026searchgeo}
Yimeng Chen, Zhe Ren, Firas Laakom, Yu~Li, Dandan Guo, and J{\"u}rgen
  Schmidhuber.
\newblock How much can we trust {LLM} search agents? measuring endorsement
  vulnerability to web content manipulation.
\newblock \emph{arXiv preprint arXiv:2606.16821}, 2026.
\newblock \doi{10.48550/arXiv.2606.16821}.
\newblock URL \url{https://arxiv.org/abs/2606.16821}.

\bibitem[Chiang \& Lee(2024)Chiang and Lee]{chiang-lee-2024-metadata}
Cheng-Han Chiang and Hung-yi Lee.
\newblock Do metadata and appearance of the retrieved webpages affect {LLM}'s
  reasoning in retrieval-augmented generation?
\newblock In \emph{Proceedings of the 7th BlackboxNLP Workshop: Analyzing and
  Interpreting Neural Networks for {NLP}}, pp.\  389--406, 2024.
\newblock \doi{10.18653/v1/2024.blackboxnlp-1.24}.
\newblock URL \url{https://aclanthology.org/2024.blackboxnlp-1.24/}.

\bibitem[Ge et~al.(2025)Ge, Wu, Chin, Lee, and Cao]{ge2025resolving}
Ziyu Ge, Yuhao Wu, Daniel Wai~Kit Chin, Roy Ka-Wei Lee, and Rui Cao.
\newblock Resolving conflicting evidence in automated fact-checking: A study on
  retrieval-augmented {LLM}s.
\newblock \emph{arXiv preprint arXiv:2505.17762}, 2025.
\newblock URL \url{https://arxiv.org/abs/2505.17762}.

\bibitem[Jiang et~al.(2023)]{jiang2023mistral}
Albert~Q. Jiang et~al.
\newblock Mistral {7B}.
\newblock \emph{arXiv preprint arXiv:2310.06825}, 2023.
\newblock URL \url{https://arxiv.org/abs/2310.06825}.

\bibitem[Jin et~al.(2024)Jin, Cao, Chen, Liu, Jiang, Xu, Qiuxia, and
  Zhao]{jin-etal-2024-tug}
Zhuoran Jin, Pengfei Cao, Yubo Chen, Kang Liu, Xiaojian Jiang, Jiexin Xu,
  Li~Qiuxia, and Jun Zhao.
\newblock Tug-of-war between knowledge: Exploring and resolving knowledge
  conflicts in retrieval-augmented language models.
\newblock In \emph{Proceedings of the 2024 Joint International Conference on
  Computational Linguistics, Language Resources and Evaluation ({LREC}-{COLING}
  2024)}, pp.\  16867--16878, 2024.
\newblock URL \url{https://aclanthology.org/2024.lrec-main.1466/}.

\bibitem[Li et~al.(2025)Li, Guo, Gao, Chen, Zhao, Zhang, Liu, Wu, Yao, and
  Wei]{li-etal-2025-llms-trust}
Yuxuan Li, Xinwei Guo, Jiashi Gao, Guanhua Chen, Xiangyu Zhao, Jiaxin Zhang,
  Quanying Liu, Haiyan Wu, Xin Yao, and Xuetao Wei.
\newblock {LLM}s trust humans more, that's a problem! unveiling and mitigating
  the authority bias in retrieval-augmented generation.
\newblock In \emph{Proceedings of the 63rd Annual Meeting of the Association
  for Computational Linguistics (Volume 1: Long Papers)}, pp.\  28844--28858,
  2025.
\newblock \doi{10.18653/v1/2025.acl-long.1400}.
\newblock URL \url{https://aclanthology.org/2025.acl-long.1400/}.

\bibitem[Liao(2026)]{liao2026auditing}
Junchi Liao.
\newblock Auditing provenance sensitivity in {LLM} agent action selection.
\newblock \emph{arXiv preprint arXiv:2607.20827}, 2026.
\newblock URL \url{https://arxiv.org/abs/2607.20827}.

\bibitem[Longpre et~al.(2021)Longpre, Perisetla, Chen, Ramesh, DuBois, and
  Singh]{longpre-etal-2021-entity}
Shayne Longpre, Kartik Perisetla, Anthony Chen, Nikhil Ramesh, Chris DuBois,
  and Sameer Singh.
\newblock Entity-based knowledge conflicts in question answering.
\newblock In \emph{Proceedings of the 2021 Conference on Empirical Methods in
  Natural Language Processing}, pp.\  7052--7063, 2021.
\newblock \doi{10.18653/v1/2021.emnlp-main.565}.
\newblock URL \url{https://aclanthology.org/2021.emnlp-main.565/}.

\bibitem[{Microsoft}(2025)]{microsoft2025phi4mini}
{Microsoft}.
\newblock {Phi-4-Mini} technical report: Compact yet powerful multimodal
  language models via mixture-of-{LoRA}s.
\newblock \emph{arXiv preprint arXiv:2503.01743}, 2025.
\newblock URL \url{https://arxiv.org/abs/2503.01743}.

\bibitem[Naphade(2026)]{naphade-2026-rational}
Atharv Naphade.
\newblock Rational synthesizers or heuristic followers? analyzing {LLM}s in
  {RAG}-based question-answering.
\newblock In \emph{Findings of the Association for Computational Linguistics:
  ACL 2026}, pp.\  40293--40311, 2026.
\newblock \doi{10.18653/v1/2026.findings-acl.2003}.
\newblock URL \url{https://aclanthology.org/2026.findings-acl.2003/}.

\bibitem[Rahadi(2026)]{rahadi2026counting}
Irwan Rahadi.
\newblock Counting copies as evidence: Confidence inflation from dependent
  evidence in retrieval-augmented generation ({RAG}), August 2026.
\newblock URL \url{https://doi.org/10.5281/zenodo.21923648}.
\newblock Position paper and preprint.

\bibitem[Reimers \& Gurevych(2019)Reimers and
  Gurevych]{reimers-gurevych-2019-sentence}
Nils Reimers and Iryna Gurevych.
\newblock Sentence-{BERT}: Sentence embeddings using {S}iamese {BERT}-networks.
\newblock In \emph{Proceedings of the 2019 Conference on Empirical Methods in
  Natural Language Processing and the 9th International Joint Conference on
  Natural Language Processing}, pp.\  3982--3992. Association for Computational
  Linguistics, 2019.
\newblock \doi{10.18653/v1/D19-1410}.
\newblock URL \url{https://aclanthology.org/D19-1410/}.

\bibitem[Ross et~al.(2026)Ross, Koopman, van~der Vegt, and
  Zuccon]{ross2026redundancy}
Jonathan~J. Ross, Bevan Koopman, Anton van~der Vegt, and Guido Zuccon.
\newblock How retriever redundancy and diversity impact {RAG} effectiveness.
\newblock \emph{arXiv preprint arXiv:2608.13956}, 2026.
\newblock URL \url{https://arxiv.org/abs/2608.13956}.

\bibitem[Schuster et~al.(2026)Schuster, Gautam, and Markert]{schuster2026whose}
Jakob Schuster, Vagrant Gautam, and Katja Markert.
\newblock Whose facts win? {LLM} source preferences under knowledge conflicts.
\newblock In \emph{Proceedings of the 64th Annual Meeting of the Association
  for Computational Linguistics (Volume 1: Long Papers)}, pp.\  29430--29459,
  2026.
\newblock \doi{10.18653/v1/2026.acl-long.1357}.
\newblock URL \url{https://aclanthology.org/2026.acl-long.1357/}.

\bibitem[Song et~al.(2026)Song, Yin, Hu, and Wang]{song2026aggregation}
Dachuan Song, Junyu Yin, Zechen Hu, and Xuan Wang.
\newblock Mergeable model-side aggregation states for long-context language
  models.
\newblock \emph{arXiv preprint arXiv:2607.26448}, 2026.
\newblock URL \url{https://arxiv.org/abs/2607.26448}.

\bibitem[Sun et~al.(2026)Sun, Bai, and Dredze]{sun-etal-2026-task}
Kaiser Sun, Fan Bai, and Mark Dredze.
\newblock Task matters: Knowledge requirements shape {LLM} responses to
  context--memory conflict.
\newblock In \emph{Findings of the Association for Computational Linguistics:
  ACL 2026}, pp.\  4154--4176, 2026.
\newblock \doi{10.18653/v1/2026.findings-acl.202}.
\newblock URL \url{https://aclanthology.org/2026.findings-acl.202/}.

\bibitem[Wallace et~al.(2024)Wallace, Xiao, Leike, Weng, Heidecke, and
  Beutel]{wallace2024instruction}
Eric Wallace, Kai Xiao, Reimar Leike, Lilian Weng, Johannes Heidecke, and Alex
  Beutel.
\newblock The instruction hierarchy: Training {LLM}s to prioritize privileged
  instructions.
\newblock \emph{arXiv preprint arXiv:2404.13208}, 2024.
\newblock URL \url{https://arxiv.org/abs/2404.13208}.

\bibitem[Wan et~al.(2024)Wan, Wallace, and Klein]{wan-etal-2024-evidence}
Alexander Wan, Eric Wallace, and Dan Klein.
\newblock What evidence do language models find convincing?
\newblock In \emph{Proceedings of the 62nd Annual Meeting of the Association
  for Computational Linguistics (Volume 1: Long Papers)}, pp.\  7468--7484,
  2024.
\newblock \doi{10.18653/v1/2024.acl-long.403}.
\newblock URL \url{https://aclanthology.org/2024.acl-long.403/}.

\bibitem[Yang et~al.(2025)]{yang2025qwen3}
An~Yang et~al.
\newblock {Qwen3} technical report.
\newblock \emph{arXiv preprint arXiv:2505.09388}, 2025.
\newblock URL \url{https://arxiv.org/abs/2505.09388}.

\bibitem[Zhan et~al.(2026)Zhan, Zhou, Li, Jing, Li, and
  Haddadi]{zhan-etal-2026-adversarial}
Zhonghao Zhan, Huichi Zhou, Zhenhao Li, Peiyuan Jing, Krinos Li, and Hamed
  Haddadi.
\newblock How adversarial environments mislead agentic {AI}?
\newblock In \emph{Findings of the Association for Computational Linguistics:
  ACL 2026}, pp.\  10264--10280, 2026.
\newblock \doi{10.18653/v1/2026.findings-acl.499}.
\newblock URL \url{https://aclanthology.org/2026.findings-acl.499/}.

\bibitem[Zhou et~al.(2024)Zhou, Li, Meng, Jiao, Ji, and
  Han]{zhou2024establishing}
Sizhe Zhou, Sha Li, Yu~Meng, Yizhu Jiao, Heng Ji, and Jiawei Han.
\newblock Establishing knowledge preference in language models.
\newblock \emph{arXiv preprint arXiv:2407.13048}, 2024.
\newblock URL \url{https://arxiv.org/abs/2407.13048}.

\bibitem[Zhou et~al.(2023)Zhou, Zhang, Poon, and Chen]{zhou-etal-2023-context}
Wenxuan Zhou, Sheng Zhang, Hoifung Poon, and Muhao Chen.
\newblock Context-faithful prompting for large language models.
\newblock In \emph{Findings of the Association for Computational Linguistics:
  EMNLP 2023}, pp.\  14544--14556, 2023.
\newblock \doi{10.18653/v1/2023.findings-emnlp.968}.
\newblock URL \url{https://aclanthology.org/2023.findings-emnlp.968/}.

\end{thebibliography}
\bibliographystyle{iclr2027_conference}

\newpage
\appendix
\raggedbottom
\section{Supporting Two-Stage Audit Theory}
\label{app:proofs}

\subsection{Finite-State Two-Stage Identity}

For item $i$, let balanced direction $S_i\in\{-1,+1\}$ determine which abstract value wins the supplied-group count.  Let $Z_i^{\mathrm{parse}}\in\mathcal Z$ be an explicit parse-stage output and $\mu_i(z)$ the mean response of a separate use-stage call when a program supplies state $z$.  Define
\[
\pi_i(z)=\frac{1}{2}\left[
\Pr(Z_i^{\mathrm{parse}}=z\mid S_i=+1,i)
-\Pr(Z_i^{\mathrm{parse}}=z\mid S_i=-1,i)
\right].
\]

\begin{theorem}[Finite-state two-stage identity]
\label{thm:two-stage-identity}
If the executed two-stage program satisfies
$\mathbb E[Y_i^{\mathrm{cas}}\mid Z_i^{\mathrm{parse}}=z,S_i=s,i]=\mu_i(z)$,
then its half contrast is
\begin{equation}
\begin{aligned}
r_i^{\mathrm{cas}}
&=\frac{\mathbb E[Y_i^{\mathrm{cas}}\mid S_i=+1,i]
-\mathbb E[Y_i^{\mathrm{cas}}\mid S_i=-1,i]}{2}\\
&=\sum_{z\in\mathcal Z}\mu_i(z)\pi_i(z).
\end{aligned}
\label{eq:two-stage-identity}
\end{equation}
\end{theorem}

\begin{proof}[Proof of Theorem~\ref{thm:two-stage-identity}]
For each $s\in\{-1,+1\}$, the law of total expectation and the explicit two-stage condition give
\[
\mathbb E[Y_i^{\mathrm{cas}}\mid S_i=s,i]
=\sum_{z\in\mathcal Z}\mu_i(z)
\Pr(Z_i^{\mathrm{parse}}=z\mid S_i=s,i).
\]
Subtracting the two directions and dividing by two gives Equation~\eqref{eq:two-stage-identity}.  The derivation accommodates dependent stages and direction-asymmetric errors.
\end{proof}

For a binary parse state, any use response has the form $\mu_i(z)=a_i+u_i z$.  If
\[
p_i=\frac{\mathbb E[Z_i^{\mathrm{parse}}\mid S_i=+1,i]
-\mathbb E[Z_i^{\mathrm{parse}}\mid S_i=-1,i]}{2},
\]
then $r_i^{\mathrm{cas}}=p_i u_i$.  Across items,
$\mathbb E[p_i u_i]=\mathbb E[p_i]\mathbb E[u_i]+\operatorname{Cov}(p_i,u_i)$.
Exact reconstruction from pooled parse and use means therefore requires control of the item-level covariance.

\subsection{Two-World Transport Lower Bound}

\begin{proposition}[Prompt-transport lower bound]
\label{prop:no-free-transport}
For responses bounded in $[-L,L]$ under arbitrary relationships between auxiliary and end-to-end prompts, two end-to-end worlds can have identical auxiliary distributions while their end-to-end effects are opposite.  Every predictor determined by the auxiliary distribution then has worst-case absolute error at least $L$.
\end{proposition}

\begin{proof}[Proof of Proposition~\ref{prop:no-free-transport}]
Fix any auxiliary-task distribution and any predictor $\widehat R$ measurable with respect to it.  Construct two compatible end-to-end worlds with the same auxiliary observations:
\[
\mathcal W_+:Y^{\mathrm{e2e}}=LS,
\qquad
\mathcal W_-:Y^{\mathrm{e2e}}=-LS.
\]
Their end-to-end effects are $+L$ and $-L$.  The predictor has the same value in both worlds, while
\[
2L\leq |\widehat R-L|+|\widehat R+L|.
\]
At least one error is at least $L$.
\end{proof}

\subsection{Common-State Three-Defect Bound}

Let $Z^{\mathrm{e2e}}$ be a reference end-to-end state on the same finite state space, let $h_i(z)$ be its reference downstream response, and define
\[
R_{\mathrm{e2e}}=\mathbb E[S_iY_i^{\mathrm{e2e}}],
\qquad
R_{\mathrm{cas}}=\mathbb E[S_i\mu_i(Z_i^{\mathrm{parse}})].
\]
Introduce
\begin{align*}
\eta_{\mathrm{dir}}
&=\left|\mathbb E[S_i\{Y_i^{\mathrm{e2e}}-h_i(Z_i^{\mathrm{e2e}})\}]\right|,\\
\eta_{\mathrm{use}}
&=\mathbb E\left[|h_i(Z_i^{\mathrm{e2e}})-\mu_i(Z_i^{\mathrm{e2e}})|\right],\\
\eta_{\mathrm{parse}}
&=\mathbb E\left[|\mu_i(Z_i^{\mathrm{e2e}})-\mu_i(Z_i^{\mathrm{parse}})|\right].
\end{align*}
Adding and subtracting the two intermediate responses and applying the triangle inequality gives
\[
|R_{\mathrm{e2e}}-R_{\mathrm{cas}}|
\leq \eta_{\mathrm{dir}}+\eta_{\mathrm{use}}+\eta_{\mathrm{parse}}.
\]
On a common reference state space, the transport gap decomposes into the three displayed defects.  If direct mismatch is zero, use mismatch is uniformly at most $\epsilon_u$, and the parse states disagree with probability at most $\epsilon_p$, with responses in $[-L,L]$, then
\[
|R_{\mathrm{e2e}}-R_{\mathrm{cas}}|
\leq \epsilon_u+2L\epsilon_p.
\]
\section{Canonical Evidence Quotient and External Grouping Guarantees}
\label{app:grouping-guarantees}
\label{app:ceq-proofs}

\subsection{Text-Only Provenance Recovery Boundary}
\label{app:text-only-proof}

\begin{proposition}[Exact boundary for deterministic text-only provenance recovery]
\label{prop:text-only-provenance}
Fix a claim $c$ and an integer $n\geq2$.  Let $\Pi_n$ be the set of partitions of $[n]$ and let $\mathcal W_{c,n}$ be the class of admissible worlds with $n$ observed records.  Let $Z_c:\mathcal W_{c,n}\to\mathcal Z^n$ map each world to its complete model-visible text observation, and let $P_c^\star:\mathcal W_{c,n}\to\Pi_n$ map each world to its oracle claim-relative provenance partition.  There exists a deterministic text-only rule $\Gamma_c:\mathcal Z^n\to\Pi_n$ satisfying $\Gamma_c(Z_c(w))=P_c^\star(w)$ for every $w\in\mathcal W_{c,n}$ if and only if, for all $w_0,w_1\in\mathcal W_{c,n}$,
\[
Z_c(w_0)=Z_c(w_1)\quad\Longrightarrow\quad P_c^\star(w_0)=P_c^\star(w_1).
\]
\end{proposition}

\begin{proof}[Proof of Proposition~\ref{prop:text-only-provenance}]
For necessity, suppose a universally exact rule $\Gamma_c$ exists.  If $Z_c(w_0)=Z_c(w_1)$, then functionality gives $\Gamma_c(Z_c(w_0))=\Gamma_c(Z_c(w_1))$, while exactness identifies the two sides with $P_c^\star(w_0)$ and $P_c^\star(w_1)$.  Hence $P_c^\star(w_0)=P_c^\star(w_1)$.

For sufficiency, suppose $P_c^\star$ is constant on each observational fiber of $Z_c$.  For every $z\in Z_c(\mathcal W_{c,n})$, define $\Gamma_c(z)$ as the unique oracle-partition value attained on the fiber $Z_c^{-1}(\{z\})$.  Fiber constancy makes this definition well-defined.  Since $\Pi_n$ is nonempty, extend $\Gamma_c$ arbitrarily outside $Z_c(\mathcal W_{c,n})$.  Then $\Gamma_c(Z_c(w))=P_c^\star(w)$ for every $w\in\mathcal W_{c,n}$.
\end{proof}

\subsection{Operational Quotient and Factorization}

Let $\mathsf{Cont}=\operatorname{Fin}^+(\mathcal Z_{\rm ev})$ be the nonempty finite canonical-content sets and let
\[
\mathsf{Quot}=\mathsf M_{\rm fin}(\mathsf{Cont})
=\{m:\mathsf{Cont}\to\mathbb N_0:\operatorname{supp}(m)\text{ is finite}\}.
\]
For $x=((k_i,z_i))_{i=1}^n\in\mathsf{Cfg}$, write $K_x$ for its occupied labels and $V_x(k)=\{z_i:k_i=k\}$.  Equation~\eqref{eq:ceq-map} maps $x$ to the outer multiset of these inner sets.  Label equality encodes membership; label spelling is removed by the quotient.

Identity-sensitive representations use an explicit authenticated signature coordinate.  Let $\Sigma$ be a nonempty space of evidence-bearing group signatures and define
\[
\mathsf{Cfg}_{\Sigma}=\{(x,s_x):x\in\mathsf{Cfg},\ s_x:K_x\to\Sigma\},
\qquad
\mathsf{Quot}_{\Sigma}=\mathsf M_{\rm fin}(\Sigma\times\mathsf{Cont}),
\]
\begin{equation}
q_{\rm ev}^{\Sigma}(x,s_x)
=\sum_{k\in K_x}\delta_{(s_x(k),V_x(k))}.
\label{eq:ceq-signature-map}
\end{equation}
Declare $(x,s_x)\sim_{\Sigma}(y,s_y)$ exactly when $q_{\rm ev}^{\Sigma}(x,s_x)=q_{\rm ev}^{\Sigma}(y,s_y)$.  A bijective label renaming $\beta:K_x\to K'$ sends $((k_i,z_i)_i,s_x)$ to $((\beta(k_i),z_i)_i,s')$, where $s'(\beta(k))=s_x(k)$.  Thus label spelling is removed while each signature--content pair is retained.  Because $\mathcal K$ is infinite, $q_{\rm ev}^{\Sigma}$ is surjective.  Consequently, for every $\Psi_{\Sigma}:\mathsf{Cfg}_{\Sigma}\to Y$, invariance on $\sim_{\Sigma}$ classes is equivalent to a unique factorization $\Psi_{\Sigma}=\bar\Psi_{\Sigma}\circ q_{\rm ev}^{\Sigma}$.  The sufficiency, faithfulness, and prompt-visible statements below have the same typed analogues.  For $a:\Sigma\times\mathsf{Cont}\to\mathcal X$, the additive analogue is
\[
E_a^{\Sigma}(x,s_x)=\sum_{k\in K_x}a(s_x(k),V_x(k)),
\qquad
\Phi_a^{\Sigma}(m)=\sum_{(\sigma,A)}m(\sigma,A)a(\sigma,A),
\]
and its faithfulness criterion is finite integer-linear independence over $(\sigma,A)\in\Sigma\times\mathsf{Cont}$.

\begin{proposition}[Operational quotient]
\label{prop:ceq-operational}
The map $q_{\rm ev}:\mathsf{Cfg}\to\mathsf{Quot}$ is surjective.  Equality of its outputs is exactly the equivalence relation generated by record permutation, bijective renaming of occupied opaque labels, and insertion or deletion of a within-group canonical duplicate while one occurrence remains.  A nontrivial split of one occupied group into $r\geq2$ nonempty groups, or a nontrivial merge of $r\geq2$ occupied groups, changes the quotient state even when the union of visible canonical content is unchanged.
\end{proposition}

\begin{proof}
Each declared nuisance transformation leaves the outer multiset of inner content sets unchanged.  Conversely, if $q_{\rm ev}(x)=q_{\rm ev}(y)$, equality of finite multisets gives a bijection between occupied labels whose matched groups have identical content sets.  Rename labels by that bijection, delete duplicate occurrences inside each matched group, and permute the remaining records; the reduced configurations coincide.  Reversing duplicate deletions recovers the originals, so the stated moves generate the whole equivalence relation.

For surjectivity, write any $m\in\mathsf{Quot}$ as $m=\sum_{j=1}^r\delta_{A_j}$.  Choose $r$ distinct labels and list one pair $(k_j,z)$ for each $z\in A_j$; the resulting configuration maps to $m$.  The empty configuration maps to the empty multiset.  Finally define $N(m)=\sum_A m(A)$.  A split changes $N$ by $r-1$ and a merge by $1-r$; both operations therefore change the quotient state.
\end{proof}

\begin{proof}[Proof of Theorem~\ref{prop:ceq-universal-factorization}]
Proposition~\ref{prop:ceq-operational} identifies the three-operation invariance with constancy on quotient fibers.  If $\Psi=\bar\Psi\circ q_{\rm ev}$ and $x\sim y$, equality of the quotient states gives $\Psi(x)=\Psi(y)$.  Conversely, suppose $\Psi$ is invariant.  For $m\in\mathsf{Quot}$ choose any $x_m$ with $q_{\rm ev}(x_m)=m$ and set $\bar\Psi(m)=\Psi(x_m)$.  Invariance makes this definition independent of the representative, and $\Psi=\bar\Psi\circ q_{\rm ev}$.  Surjectivity of $q_{\rm ev}$ makes $\bar\Psi$ unique.

If $\Psi$ is invariant, every composition $h\circ\Psi$ is invariant.  Conversely, quantifying over all downstream maps includes the identity map on $Y$; equivalently, binary maps can separate any two unequal points of $Y$.  For a specified downstream family, point separation on $\Psi(\mathsf{Cfg})$ supplies the converse.
The final faithfulness and recoverability equivalences follow from Proposition~\ref{prop:ceq-minimal-sufficiency} below, whose proof uses only the factorization just established.
\end{proof}

\subsection{Faithfulness, Minimal Sufficiency, and Additive Realizations}

An enforcing representation $\Theta:\mathsf{Cfg}\to W$ is \emph{quotient-faithful} when $\Theta(x)=\Theta(y)$ if and only if $x\sim y$.  It is \emph{universally sufficient} when, for every set $D$, every relation-invariant observable $f:\mathsf{Cfg}\to D$ factors through $\Theta$.  Together, quotient faithfulness and universal sufficiency characterize exact preservation of relation-invariant information.

\begin{proposition}[Minimal sufficiency and faithful recoding]
\label{prop:ceq-minimal-sufficiency}
The quotient $q_{\rm ev}$ is universally sufficient.  If $\Theta$ is universally sufficient, then $q_{\rm ev}=\rho\circ\Theta$ for some $\rho:W\to\mathsf{Quot}$.  If $\Theta=\bar\Theta\circ q_{\rm ev}$ is enforcing, then the following are equivalent: $\Theta$ is universally sufficient; $\Theta$ is quotient-faithful; and $\bar\Theta$ is injective on $\mathsf{Quot}$.
\end{proposition}

\begin{proof}
Theorem~\ref{prop:ceq-universal-factorization} factors every invariant $f$ through $q_{\rm ev}$, proving its universal sufficiency.  If $\Theta$ is universally sufficient, apply its definition to the invariant observable $f=q_{\rm ev}$ to obtain $q_{\rm ev}=\rho\circ\Theta$.  Hence $\Theta(x)=\Theta(y)$ implies $q_{\rm ev}(x)=q_{\rm ev}(y)$.  Enforcement supplies the reverse implication, proving faithfulness.

For an enforcing $\Theta=\bar\Theta\circ q_{\rm ev}$, faithfulness implies injectivity of $\bar\Theta$ by choosing representatives of any two quotient states.  Conversely, injectivity of $\bar\Theta$ makes equality of $\Theta$ equivalent to equality of $q_{\rm ev}$.  The inverse of $\bar\Theta$ on the reachable image recovers $q_{\rm ev}$, so every invariant observable factors through $\Theta$; extend that factor arbitrarily outside the reachable image.
\end{proof}

Let $a:\mathsf{Cont}\to\mathcal X$ be a deterministic group aggregator and define
\begin{equation}
E_a(x)=\sum_{k\in K_x}a(V_x(k)),
\qquad
\Phi_a(m)=\sum_{A\in\mathsf{Cont}}m(A)a(A).
\label{eq:ceq-additive}
\end{equation}

\begin{proposition}[Additive faithfulness criterion]
\label{prop:ceq-additive-faithfulness}
Every additive realization satisfies $E_a=\Phi_a\circ q_{\rm ev}$ and therefore enforces the declared nuisance invariance.  It is quotient-faithful if and only if the family $\{a(A):A\in\mathsf{Cont}\}$ is finitely integer-linearly independent:
\[
\sum_{A\in\mathsf{Cont}}c(A)a(A)=0,\qquad
c\in\mathbb Z^{(\mathsf{Cont})}
\quad\Longrightarrow\quad c=0,
\]
where the parenthesized exponent denotes finite support.  A faithful realization always exists by taking $\mathcal X=\mathbb R^{(\mathsf{Cont})}$ and $a(A)=e_A$, the canonical coordinate vector.
\end{proposition}

\begin{proof}
Grouping equal outer multiset atoms in Equation~\eqref{eq:ceq-additive} gives $E_a(x)=\Phi_a(q_{\rm ev}(x))$.  If a nonzero integer relation exists, decompose $c=c^+-c^-$ into distinct nonnegative finite multisets.  Then $\Phi_a(c^+)=\Phi_a(c^-)$, and surjectivity supplies two distinct quotient states that collide under $E_a$.  Conversely, any collision $\Phi_a(m)=\Phi_a(m')$ for $m\neq m'$ yields the nonzero integer relation $c=m-m'$.  The canonical coordinate vectors are integer-linearly independent and expose every multiset coefficient, so they give a faithful realization.
\end{proof}

\begin{corollary}[Prompt-visible enforcement criterion]
\label{cor:ceq-prompt-boundary}
Let $R:\mathsf{Cfg}\to\mathcal P$ be the actual model-visible renderer.  Uniform invariance of $h\circ R$ for every binary downstream map $h:\mathcal P\to\{0,1\}$ holds if and only if $R=\bar R\circ q_{\rm ev}$.  Therefore, if $x\sim y$ but $R(x)\neq R(y)$, a deterministic binary downstream rule exists that distinguishes the two rendered prompts.
\end{corollary}

\begin{proof}
Apply Theorem~\ref{prop:ceq-universal-factorization} to $R$.  Any pair of distinct rendered prompts is separated by a binary map, establishing necessity of the factorization criterion.
\end{proof}

The corollary characterizes uniform invariance over all binary downstream maps; fixed-model response is measured empirically.  The default quotient carries content and group multiplicity, while Equation~\eqref{eq:ceq-signature-map} additionally preserves authenticated identity, credibility, or reliability metadata.  Opaque membership labels encode membership in both versions.

\subsection{Concrete Additive Construction and Grouping-Error Bound}

Let record $r$ support target $T$ or competitor $C$, and let $g(r)$ denote its supplied group.  With group sets $G_T$ and $G_C$, a unit-weight reference margin is
\begin{equation}
    D = b + |G_T|-|G_C|,
    \label{eq:rootmargin}
\end{equation}
where $b$ is the pre-context preference for $T$.  Copying a record while preserving the group set leaves $D$ unchanged.  An exact-replication-invariant representation preserves the evidence state when another canonical copy is added to its existing group.  The design separates evidential multiplicity from within-group content aggregation: unique content is aggregated into a bounded group representation, and each group contributes once.  We realize this property both with an exact stable-representative construction and with exact deduplication followed by lossless within-group concatenation.  Our mirrored topology gives the favored side two groups and the other side one group, then reverses that assignment while holding the records fixed.

We now give the complete guarantees for the external representation in Equation~\eqref{eq:grouped-evidence}.  The ambient record universe $U$, canonicalizer $z_{\rm ev}:U\to\mathcal Z_{\rm ev}$, and deterministic aggregator $a$ are fixed across all compared partitions, and $a$ receives the extensional set of canonical group elements.

\begin{proof}[Proof of Proposition~\ref{prop:external-replication-invariance}]
Let $G\in P$ be the block containing $i$, set $I'=I\cup\{i'\}$, and define
\[
P'=(P\setminus\{G\})\cup\{G\cup\{i'\}\}.
\]
Because $z_{\rm ev}(i')=z_{\rm ev}(i)$,
\[
\{z_{\rm ev}(j):j\in G\cup\{i'\}\}=\{z_{\rm ev}(j):j\in G\}.
\]
The changed block therefore supplies exactly the same set to the fixed aggregator $a$, while every other block is unchanged.  The two finite sums agree term by term, so $E_{I'}(P')=E_I(P)$.
\end{proof}

For two partitions $P$ and $Q$ of the same finite record set $I$, build a bipartite overlap graph whose vertices are their blocks and whose edges join blocks with nonempty record intersection.  Exclude exactly those connected components containing one $P$-block and one $Q$-block with identical record sets.  Let $\mathcal C_{\mathrm{err}}$ denote the remaining components; for $c\in\mathcal C_{\mathrm{err}}$, let $p_c$ and $q_c$ be its numbers of $P$- and $Q$-blocks, and define
\[
H(P,Q)=\sum_{c\in\mathcal C_{\mathrm{err}}}(p_c+q_c).
\]

\begin{proposition}[Partition-error bound]
\label{prop:partition-error-bound}
Suppose $B\geq0$ and $\|a(S)\|\leq B$ for every semantic set that occurs under $P$ or $Q$.  Then
\[
\|E_I(P)-E_I(Q)\|\leq B H(P,Q).
\]
\end{proposition}

\begin{proof}
Every excluded component contributes the same record block under both partitions.  The same $z_{\rm ev}$ and $a$ therefore produce identical vectors, which cancel.  In an erroneous component $c$, write $P_c$ and $Q_c$ for its blocks.  The triangle inequality and the assumed bound give
\begin{align*}
\left\|\sum_{G\in P_c}a(\{z_{\rm ev}(i):i\in G\})
-\sum_{G\in Q_c}a(\{z_{\rm ev}(i):i\in G\})\right\|
&\leq \sum_{G\in P_c}\|a(\{z_{\rm ev}(i):i\in G\})\| \\
&\quad +\sum_{G\in Q_c}\|a(\{z_{\rm ev}(i):i\in G\})\| \\
&\leq B(p_c+q_c).
\end{align*}
The erroneous components are disjoint.  Summing their differences and applying the triangle inequality once more yields $\|E_I(P)-E_I(Q)\|\leq B H(P,Q)$.
\end{proof}

The coefficient one is worst-case sharp over this admissible class.  For any $B>0$, take $U=I=\{1,2\}$, $\mathcal X=\mathbb R$, and $\mathcal Z_{\rm ev}=\{z_1,z_2\}$ with $z_1\neq z_2$, where $z_{\rm ev}(1)=z_1$ and $z_{\rm ev}(2)=z_2$.  Define the fixed aggregator on every subset of $\mathcal Z_{\rm ev}$ by
\[
a(\varnothing)=0,\qquad
a(\{z_1\})=a(\{z_2\})=B,\qquad
a(\{z_1,z_2\})=-B.
\]
For $P=\{\{1\},\{2\}\}$ and $Q=\{\{1,2\}\}$, the overlap graph has one erroneous component with $H(P,Q)=3$, while $E_I(P)=2B$ and $E_I(Q)=-B$.  Hence $\|E_I(P)-E_I(Q)\|=3B=B H(P,Q)$.  This witness establishes worst-case sharpness over the allowed class.

\begin{proposition}[Sharp content-aware partition-error bound]
\label{prop:content-aware-error}
Fix a finite record set $I$, its canonicalizer, and partitions $P,Q$ of $I$.  For a fixed aggregator $a$ with $\|a(A)\|\le B$ on every occurring content set,
\[
\|E_I(P)-E_I(Q)\|\le B\Delta_q(P,Q)\le BH(P,Q).
\]
For each fixed pair $P,Q$ and $B\ge0$,
\[
\sup_{a:\,|a(A)|\le B}
\left|\sum_A(q_P(A)-q_Q(A))a(A)\right|
=B\Delta_q(P,Q),
\]
where the supremum is over scalar aggregators on finite nonempty canonical-content sets.  An $L$-Lipschitz downstream margin therefore obeys $|m(E_I(P))-m(E_I(Q))|\le LB\Delta_q(P,Q)$, and its nonzero sign is stable if $|m(E_I(P))|>LB\Delta_q(P,Q)$.
\end{proposition}

\begin{proof}
Put $c_A=q_P(A)-q_Q(A)$.  The fixed canonicalizer and fixed aggregator give
\[
E_I(P)-E_I(Q)=\sum_A c_Aa(A).
\]
The triangle inequality bounds the norm by $B\sum_A|c_A|$.  Identical record blocks cancel between the two partitions.  The total numbers of remaining blocks on the two sides sum to $H(P,Q)$; mapping these blocks to their content sets can cancel additional mass, so $\sum_A|c_A|\le H(P,Q)$.  For sharpness, choose the single fixed scalar function $a(A)=B\operatorname{sgn}(c_A)$, with $\operatorname{sgn}(0)=0$ and zero on absent content sets.  It attains $B\sum_A|c_A|$.  This is a worst-case choice for the given pair; a particular implemented aggregator can have a smaller distortion.  Lipschitz continuity and the strict margin argument prove the final two claims.
\end{proof}

The two discrepancies distinguish record-level and canonical-content-level changes.  For example, let $z_1=z_2=u$ and $z_3=z_4=v$ with $u\ne v$, and compare $P=\{\{1,3\},\{2,4\}\}$ with $Q=\{\{1,4\},\{2,3\}\}$.  Their overlap graph has $H(P,Q)=4$, but both quotient states equal $2\delta_{\{u,v\}}$, so $\Delta_q(P,Q)=0$ and every fixed set-based aggregator gives identical states.  Conversely, the four-distinct-element witness of Proposition~\ref{prop:equal-count-separation} has $H=\Delta_q=4$ and attains distortion $4B$, despite zero change in group count.  More generally, $\Delta_q$ is a pseudometric on partitions and becomes an $\ell_1$ metric on their distinct reachable quotient states.

\begin{corollary}[Downstream decision stability]
\label{cor:grouping-decision-stability}
Let $L\geq0$ and let $m:\mathcal X\to\mathbb R$ be $L$-Lipschitz.  Then
\[
|m(E_I(P))-m(E_I(Q))|\leq LBH(P,Q).
\]
For a binary decision boundary at $m=0$, the nonzero sign is unchanged whenever
\[
|m(E_I(P))|>LBH(P,Q).
\]
\end{corollary}

\begin{proof}
The first claim follows by applying Lipschitz continuity to Proposition~\ref{prop:partition-error-bound}.  Under the displayed strict condition, the largest admissible change in the margin is smaller than its distance from zero, so the sign remains unchanged.  The displayed strict inequality provides a uniform two-sided guarantee: the perturbation remains strictly inside the reference margin.  At equality the perturbed margin can reach zero, so a fixed tie convention covers a single reference sign.
\end{proof}

These guarantees compose with an externally specified representation pipeline whose provenance signals, supplied partition, and canonicalized content provide the upstream inputs.  Authenticated signatures can carry verified identity without changing the role of opaque membership labels.  External caps on the group aggregator instantiate $B$, and an independently verified Lipschitz bound instantiates $L$.  Behavioral curves test grouping effects; system-level checks provide the numerical stability constants.

\section{Detailed Experimental Specifications}

\subsection{Cue and Topology Families}

The decision audit has five prompt families.  Marker swaps which value is labeled verified.  Graph swaps anonymous record-to-node edges.  Provenance uses the same graph with nodes explicitly described as origins.  Rule adds a same-origin-counts-once instruction.  Full crosses marker correctness and topology correctness.  Within a paired family, candidate strings, claims, record order, value-table order, literal offsets, and the full tokenizer multiset are fixed.

\subsection{Direct Pipeline Modes}

Edge asks whether two specified records share a root.  Count asks for the number of distinct roots supporting a deterministically selected abstract value.  Winner asks which value has more distinct roots.  Opaque oracle supplies the 2:1 count summary after removing the graph.  Semantic oracle supplies the same summary alongside the original fact task.  End-to-end supplies graph, rule, and fact task together.  All modes use direct semantic candidates.

\subsection{Repetition Doses}

At every dose, two records supporting the factual target arise from two distinct roots.  The competitor has one root, whose record is repeated $D$ times.  Raw hides root IDs; Prompt preserves all records while exposing the root mapping and deduplication rule.  External retains one representative per oracle root and removes root instructions from the final prompt.  This stable-representative operator certifies multiplicity control; the broader systems recommendation aggregates complementary within-root content under one bounded evidential contribution.  Forward and reverse record orders are both evaluated.  The External and Raw $D=1$ serialized inputs are identical for all 282 item-presentation-order pairs in each checkpoint.

\subsection{Natural-Text Supplied Units}
\label{app:phase8-specification}

Phase 8 freezes two domains before model inference.  \textsc{Human Evidence} contains 101 PERSPECTRUM claims grouped into 49 dependence components.  Each side supplies four different crowd-annotated stance-perspective clusters, matched one-to-one to side-exclusive evidence-unit IDs and unique normalized and visible-text hashes.  Each ID indexes one crowd-annotated stance-perspective evidence unit.  \textsc{Web Consensus} contains 138 ConflictingQA questions in 135 dependence components, with frozen two-model-consensus stance groupings.  Each selected eight-record panel has operationally separated registrable domains, canonical URLs, normalized-text hashes, and visible-text hashes across both sides.  A WEB unit is a supplied canonical-URL key under these frozen criteria.  Publisher identity and ownership enter through explicit authenticated signature coordinates.  The two domains are analyzed separately with domain-specific component bootstraps.

Text is normalized with Unicode NFKC, removal of control, format, and private-use characters, and whitespace collapse while preserving ASCII punctuation.  Eligible evidence has at least 400 normalized characters.  The model-visible field is a 400-character head--tail excerpt comprising the first 200 characters, the literal marker \texttt{ [...] }, and the last 193.  Every Base prompt contains three evidence records; Copy, Prompt, and Four-Unit contain six.  Copy and Four-Unit match record count and 2,400 visible evidence characters; Phase 11 additionally matches final character, byte, and tokenizer length exactly.

For each claim and attack side, a frozen hash selects two opposite-side anchors and one dose anchor.  \textsc{Base} contains these three units.  \textsc{Copy} repeats the dose anchor four times.  \textsc{Prompt} preserves the same six evidence texts and record count as Copy while showing short opaque aliases and a same-unit-counts-once rule.  \textsc{Four-Unit} retains the same anchors and dose anchor, replacing three copies with three other attack-side units.  The machine artifacts retain the frozen condition key \texttt{DISTINCT}.  Support and undermine attacks and original and reverse record orders yield 16 trials per claim.  Opaque aliases are fixed by canonical role before reversal; reversal preserves every field and changes the record sequence.

Candidates are the direct semantic strings \texttt{Yes} and \texttt{No}; scoring starts from each tokenizer's native assistant state and termination sequence.  The primary outcome is strict selection of the manipulated side, with an exact likelihood tie scored zero.  We average two attack sides and two orders within claim before forming contrasts.  Each domain uses its own frozen dependence-component bootstrap stream for every contrast and metric; a component joins claims that reuse a supplied key or final visible text.  The point estimand remains claim-weighted.  Ten thousand replicates use master seed 20260828 with deterministic domain-specific offsets (20260828 for WEB and 20270828 for HUMAN); auxiliary sensitivities use separately frozen offsets, and ordinary item bootstrap provides sensitivity estimates.

Copy--Base tests replication invariance: its 90\% interval must lie entirely inside $[-5,+5]$ points.  Four-Unit--Base, Four-Unit--Copy, and Copy--Prompt each require a point estimate of at least $+5$ points and a 95\% lower bound above zero.  These gates are evaluated separately for every checkpoint and domain.  For each checkpoint, 956 byte-for-byte checks show that stable collapse of Copy by its hidden unit key reproduces the corresponding Base prompt.  Consequently, upstream recovery is numerically identical to Copy--Base and represents the same evidence contrast.

The 5-point operational margin corresponds to approximately one changed choice per 20 constrained decisions.  The stricter 10-point Phase 6 marker criterion was likewise frozen before treatment inference.  The two gate types answer different questions: an equivalence band assesses whether an effect is small enough to treat as negligible, whereas a positive-effect gate assesses whether an estimate reaches a stated magnitude with direction supported by its interval.  Appendix Table~\ref{tab:threshold-sensitivity} reports a 2.5/5/7.5/10-point sensitivity sweep around the frozen gates.  Copy--Base lies outside equivalence in all eight model--domain cells at every margin; model-specific estimates and intervals are primary, and pass counts summarize gate-level magnitudes.

\subsection{Extended Checkpoint and Grouping Specifications}
\label{app:extended-specifications}

\paragraph{Added checkpoint grid.}
The complete Phase 8 grid is repeated with the same prompts, candidates, estimands, and bootstrap streams on Qwen3-14B-AWQ and Phi-4.  Qwen3-14B uses the publisher's AWQ 4-bit checkpoint; Phi-4 uses dynamic NF4 inference.  The Qwen3-8B NF4 rerun is counted as a quantization-control inference configuration under the existing Qwen3-8B checkpoint identity.  Each inference configuration completes 3,824 trials, for 11,472 new candidate-choice forward trials.  The main paper reports the two added 14B checkpoints; the control remains bound in the machine snapshot and trial accounting.

\paragraph{Controlled answer-plus-one-sentence generation.}
The generation panel freezes 40 HUMAN and 40 WEB items, one per Phase 8 dependence component, before inference.  For each item it crosses \textsc{Base}, raw four-copy, and prompt-rule conditions with two attack sides and two record orders, yielding 960 trials per checkpoint.  Qwen3-8B runs in bfloat16, Qwen3-14B uses its official AWQ checkpoint, and Phi-4 uses dynamic NF4 inference.  Greedy decoding requests a leading \texttt{Yes} or \texttt{No} followed by one concise sentence, with at most 48 new tokens.  All-trial leading-answer selection is the primary endpoint; invalid outputs receive zero attack selection.  The valid-output analysis is a frozen sensitivity.  The explanatory sentence is retained in the output, and scoring uses deterministic leading-answer parsing.

For each domain and checkpoint, we average the two attack sides and two orders within item, then bootstrap the 40 frozen dependence components 10,000 times.  All three formal runs contain 960/960 predictions and zero invalid outputs, for 2,880 controlled-generation trials.

\paragraph{False splits and false merges after fixing supplied keys.}
This audit conditions on the supplied Phase 8 grouping keys and injects partition error before record emission.  Eligibility requires both sides of a claim to contain a frozen document of at least 160 whitespace-delimited words.  The resulting HUMAN panel has 66 claims in 32 dependence components; the WEB panel retains 138 questions in 135 components.  This filter is model-independent and fixed before treatment inference.

For the one-group panel, the attack material consists of four non-overlapping 40-word windows from one frozen document.  For the four-group panel, four 40-word windows come from four supplied operational units.  Estimated keys emit one, two, or four records after within-group exact-text deduplication and a fixed per-group content budget.  The richer \textsc{Partial-G1-Full-Content} endpoint concatenates all four unique windows into one 160-word record while keeping estimated group count at one; subtracting \textsc{Partial-G1} measures the full-window inclusion gain relative to the one-window 40-word endpoint.  The main content-fixed contrasts retain all 160 attack words in the one-record condition: \textsc{Partial-G4} minus \textsc{Partial-G1-Full-Content} measures false splitting of one supplied document, and \textsc{Dist-G4} minus \textsc{Dist-G1-Full-Content} measures the loss from merging four supplied operational units.  In each pair, the ordered attack words are identical and record placement implements the grouping intervention.

The primary outcome is the attack-side likelihood share in Equation~\eqref{eq:attack-probability}.  Secondary endpoints are strict selection, half-tie selection, and likelihood margin.  Two attack sides and two orders are averaged within item, and 10,000 bootstraps resample supplied-key or final-text connected components while retaining all member claims.  Twelve conditions produce 9,792 trials per checkpoint and 39,168 across Qwen3-8B, Qwen3-4B, Phi-4-mini, and Mistral-7B.  Conditioning on the supplied operational keys, the experiment identifies the causal effect of downstream record placement.

\paragraph{Matched six-slot serialization control.}
Phase 9 fixes evidence words while jointly varying partition rendering and its deterministic serialization.  Its four-record arm contains three additional JSON objects, ordinal record IDs, side and evidence headers, and object separators.  The serialized difference is 147--153 UTF-8 bytes and 43--63 tokenizer input tokens across the four checkpoints, while the evidence-token difference is bounded by five.  Table~\ref{tab:phase9-main} therefore estimates the total effect of operational record placement under fixed evidence words.

Phase 11 removes these count and length differences.  Both arms use six JSON objects with the same R1--R6 identifiers, side labels, headers, separators, instruction, query, claim, anchor content, and candidate format.  The same ordered 160-word attack stream is placed in four attack-side slots as either $[160,0,0,0]$ or $[40,40,40,40]$ words; the empty slots and their headers remain model-visible in both arms.  A deterministic outcome-blind search adds up to four whitespace characters after the final JSON brace so that each tokenizer-specific pair has exactly the same final character count, UTF-8 byte count, and input-token count while parsing to the same JSON records.  Model-visible, parse-inert whitespace supplies the exact length match.  The identified treatment is the placement of a fixed word stream across one versus four nonempty content-bearing record slots within a matched six-slot skeleton.

Each checkpoint completes 3,264 choice trials, yielding \PhaseElevenChoiceTrials{} trials and \PhaseElevenCandidateForwardCalls{} physical candidate forward calls with zero failures.  Two sides and two orders are averaged within item; 10,000 replicates resample the same dependence components used by the parent audit.  All 16 model--domain--contrast point estimates have the predicted positive sign, and 12 intervals lie entirely above zero.  The matched control identifies a partition-boundary effect beyond extra objects, repeated headers, identifiers, separators, and total tokenizer length.

\begin{table}[H]
\caption{Phase 11 matched-serialization control. Effects are percentage-point changes in $p_{\mathrm{attack}}$ with 95\% dependence-component bootstrap intervals. The intervention is matched six-slot serialization; the same ordered evidence word stream is placed across one versus four nonempty content-bearing record slots. Model-visible, parse-inert pair padding equalizes final character, byte, and tokenizer lengths. The estimand is operational content placement across record slots within the matched skeleton.}
\label{tab:phase11-serialization-control}
\centering
\small
\setlength{\tabcolsep}{3.2pt}
\begin{tabular}{llrr}
\toprule
Model & Domain & One supplied unit: split & Four supplied units: merge loss \\
\midrule
Qwen3-8B & WEB (138/135) & +8.27 [+5.73, +10.82] & +8.00 [+5.18, +10.88] \\
Qwen3-8B & HUMAN (66/32) & +5.92 [+2.39, +8.87] & +11.27 [+7.37, +15.96] \\
Qwen3-4B & WEB (138/135) & +9.95 [+6.90, +13.04] & +13.33 [+10.29, +16.43] \\
Qwen3-4B & HUMAN (66/32) & +1.24 [-2.48, +5.69] & +7.47 [+3.57, +12.23] \\
Phi-4-mini & WEB (138/135) & +1.24 [+0.38, +2.29] & +1.25 [+0.40, +2.39] \\
Phi-4-mini & HUMAN (66/32) & +0.63 [-0.02, +1.65] & +0.99 [+0.01, +2.44] \\
Mistral-7B & WEB (138/135) & +3.22 [+1.11, +5.27] & +1.46 [-0.40, +3.34] \\
Mistral-7B & HUMAN (66/32) & +6.46 [+3.65, +8.84] & +2.33 [-0.65, +5.73] \\
\bottomrule
\end{tabular}
\end{table}

\paragraph{Controlled GEO proliferation panel.}
The frozen panel contains 48 fictional product pairs.  For each product and attack side, the same-root family uses four exactly 40-word genre-varied records bound to one three-fact campaign brief and one oracle root.  The independent-root family uses four exactly 40-word records bound respectively to laboratory, user-panel, endurance, and service-data facts generated by four oracle roots.  The source validator confirms that every visible claim belongs to its enumerated atomic-fact set, each attack stream contains exactly 160 ordered words, and hidden root, author, and synthetic-domain fields are absent from the prompt.  Both structures use the Phase 11 six-slot renderer and compare $[160,0,0,0]$ with $[40,40,40,40]$ under exact character, UTF-8 byte, and model-token length matching.

Each checkpoint completes 768/768 choices with zero failures, yielding \PhaseTwelveChoiceTrials{} choices and \PhaseTwelveCandidateForwardCalls{} candidate forward calls.  Every condition--attack-side--order cell contains 48 trials.  The primary item-first estimand averages the two attack sides and two orders before taking the contrast; 10,000 fixed-seed bootstrap replicates resample the 48 items.  A descriptive mirror decomposition reveals substantial presentation interactions.  For Qwen3-8B, the independent-root contrast is $+41.43$ points in original order and $-49.92$ in reverse order.  The main table therefore reports balanced mirrored effects for each checkpoint and preserves both reversals whose pointwise 95\% intervals exclude zero.

\paragraph{Grouping baselines and downstream distortion.}
We evaluate four grouping rules on 192 within-item partition problems containing 768 records.  Oracle root identifiers and single-link cosine clustering with all-MiniLM-L6-v2 at threshold 0.80 recover every partition.  Normalized exact hashing and 128-permutation, three-word-shingle MinHash at threshold 0.80 emit all four same-root records as singletons while preserving all four independent roots.  On the 96 same-root scopes, these lexical methods have pairwise recall and F1 equal to zero, 576 false-split pairs, adjusted Rand index zero, B-cubed F1 $0.40$, and total $H=480$.  Table~\ref{tab:phase12-grouping-baselines} reports partition accuracy, and Table~\ref{tab:phase12-downstream-distortion} maps every predicted partition to the corresponding executed intervention arm.  Thresholds were fixed before result inspection.  The embedding result establishes separability for these controlled templates; Proposition~\ref{prop:text-only-provenance} characterizes the observational-collision boundary for universal exact recovery.

\begin{table}[t]
\centering
\caption{Grouping baselines on the controlled provenance panel.  Exact partition rates are over 96 attack-side scopes per provenance structure.  $H$ is the total partition-error quantity from Proposition~\ref{prop:partition-error-bound}.}
\label{tab:phase12-grouping-baselines}
\small
\begin{tabular}{lrrr}
\toprule
Method & \shortstack{Same-root\\exact} & \shortstack{Independent-root\\exact} & Total $H$ \\
\midrule
Oracle root & 100\% & 100\% & 0 \\
Exact normalized hash & 0\% & 100\% & 480 \\
MinHash & 0\% & 100\% & 480 \\
Sentence embedding & 100\% & 100\% & 0 \\
\bottomrule
\end{tabular}
\end{table}

\begin{table}[t]
\centering
\caption{Downstream grouping distortion on same-root campaign scopes.  Each entry is $|p_{\rm attack}(\widehat P)-p_{\rm attack}(P^\star)|$ in percentage points.  The predicted partitions coincide exactly with executed matched-intervention arms.  Independent-root distortion is zero for all three methods because each preserves the four oracle roots.}
\label{tab:phase12-downstream-distortion}
\small
\begin{tabular}{lrrrr}
\toprule
Method & Qwen3-8B & Qwen3-4B & Phi-4-mini & Mistral-7B \\
\midrule
Exact normalized hash & 18.46 & 9.66 & 27.44 & 4.44 \\
MinHash & 18.46 & 9.66 & 27.44 & 4.44 \\
Sentence embedding & 0.00 & 0.00 & 0.00 & 0.00 \\
\bottomrule
\end{tabular}
\end{table}

\section{Additional Results}

\subsection{Order Decomposition of Controlled Partition Effects}
\begin{table}[H]
\centering
\caption{Post-hoc order decomposition of the existing controlled GEO predictions.  Both families use four content-bearing records minus one grouped record.  Entries are percentage points with pointwise 95\% bootstrap intervals over 48 item-components after averaging the two attack sides.  Interaction is original minus reverse, estimated with paired resampling.  Each checkpoint and family is reported separately; the intervals are descriptive.  Values smaller than 0.005 points in magnitude round to 0.00.}
\label{tab:geo-order-decomposition}
\footnotesize
\setlength{\tabcolsep}{3pt}
\begin{tabular}{llrrr}
\toprule
Model & Roots & Original order & Reverse order & Order interaction \\
\midrule
Qwen3-8B & Same root & $+16.83[+12.37,+21.54]$ & $+20.09[+15.91,+24.46]$ & $-3.26[-9.91,+3.37]$ \\
Qwen3-8B & Independent & $+41.43[+36.15,+46.22]$ & $-49.92[-49.95,-49.88]$ & $+91.36[+86.07,+96.14]$ \\
Qwen3-4B & Same root & $-13.97[-19.26,-9.39]$ & $+33.28[+29.53,+36.85]$ & $-47.26[-53.93,-41.21]$ \\
Qwen3-4B & Independent & $+10.98[+6.43,+15.92]$ & $-0.19[-0.28,-0.12]$ & $+11.17[+6.63,+16.10]$ \\
Phi-4-mini & Same root & $+17.95[+14.19,+21.99]$ & $+36.93[+33.56,+39.98]$ & $-18.98[-23.42,-14.32]$ \\
Phi-4-mini & Independent & $+19.15[+15.09,+23.29]$ & $+12.31[+9.13,+15.79]$ & $+6.84[+2.33,+11.35]$ \\
Mistral-7B & Same root & $-8.88[-14.64,-3.29]$ & $0.00[0.00,0.00]$ & $-8.88[-14.64,-3.29]$ \\
Mistral-7B & Independent & $+6.03[+3.27,+9.28]$ & $+0.01[+0.01,+0.01]$ & $+6.03[+3.26,+9.27]$ \\
\bottomrule
\end{tabular}
\end{table}

Table~\ref{tab:geo-order-decomposition} reuses all 3,072 stored Phase 12 choices.  Each model has 48 items, two root families, two attack sides, two order mirrors, and two partition arms.  For each item, family, and order, we subtract the one-group probability from the four-group probability and then average the two attack sides.  The original-minus-reverse contrast is formed within the same item.  We resample the 48 distinct item-components 10,000 times with seed 20260907, keeping every paired contrast together, and report percentile intervals.  These are post-hoc descriptive intervals; the primary balanced estimates and their original intervals remain as reported in Table~\ref{tab:phase12-geo}.  The decomposition makes the presentation dependence of each checkpoint visible alongside its balanced effect.

\subsection{Checkpoint and Generation Breadth}

\begin{table}[H]
\caption{Breadth checks at added 14B checkpoints and in controlled generation, in percentage points with 95\% dependence-component bootstrap intervals. The generation endpoint is the parsed leading Yes/No answer.}
\label{tab:phase9-breadth}
\centering
\scriptsize
\setlength{\tabcolsep}{3.0pt}
\begin{tabular}{@{}llrr@{}}
\toprule
Checkpoint & Domain & Copy--Base & Copy--Prompt \\
\midrule
\multicolumn{4}{@{}l}{\emph{Added 14B checkpoints, complete-candidate likelihood}} \\
Qwen3-14B official AWQ 4-bit & HUMAN & $+51.49[+45.71,+57.43]$ & $+4.70[+0.24,+9.23]$ \\
Qwen3-14B official AWQ 4-bit & WEB & $+30.62[+25.90,+35.40]$ & $+0.91[-2.14,+3.93]$ \\
Phi-4 14B dynamic NF4 & HUMAN & $+51.98[+45.41,+58.82]$ & $+9.16[+5.46,+12.99]$ \\
Phi-4 14B dynamic NF4 & WEB & $+43.30[+38.13,+48.52]$ & $+1.63[-1.48,+4.86]$ \\
\midrule
\multicolumn{4}{@{}l}{\emph{Controlled generation, leading-answer endpoint}} \\
Checkpoint & Domain & \multicolumn{2}{c}{Leading Copy--Base} \\
\midrule
Qwen3-8B & HUMAN & \multicolumn{2}{c}{$+20.62[+15.00,+26.25]$} \\
Qwen3-8B & WEB & \multicolumn{2}{c}{$+15.62[+9.38,+22.50]$} \\
Qwen3-14B-AWQ & HUMAN & \multicolumn{2}{c}{$+16.25[+9.38,+23.12]$} \\
Qwen3-14B-AWQ & WEB & \multicolumn{2}{c}{$+16.88[+10.62,+23.12]$} \\
Phi-4 & HUMAN & \multicolumn{2}{c}{$+20.00[+13.75,+26.88]$} \\
Phi-4 & WEB & \multicolumn{2}{c}{$+20.00[+13.12,+27.50]$} \\
\bottomrule
\end{tabular}
\end{table}

\subsection{Grouping-Key Robustness Curves}

\begin{figure}[H]
\centering
\includegraphics[width=\textwidth]{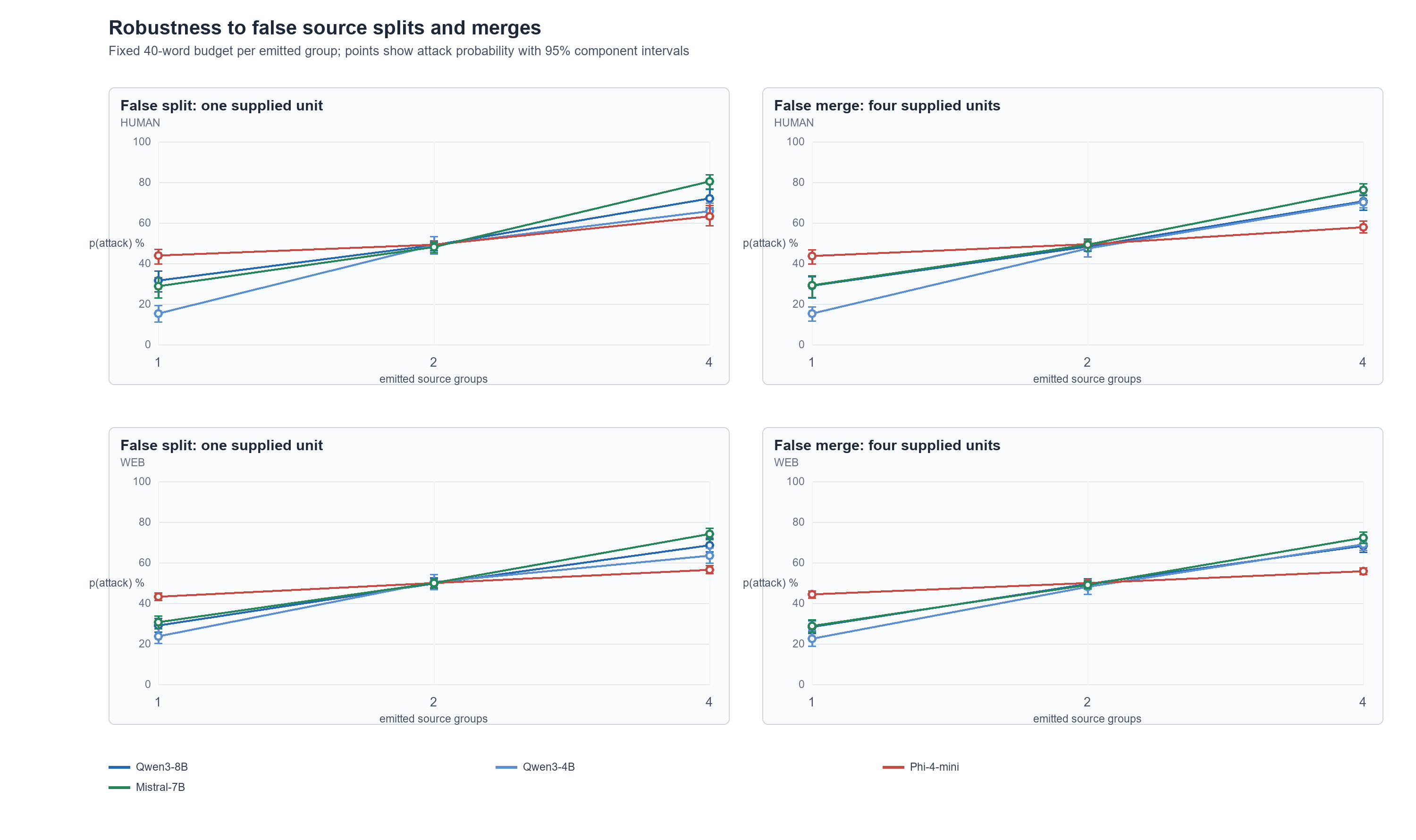}
\caption{Fixed-budget supplied-key robustness curves. The partial-copy panels split non-overlapping 40-word windows from one frozen document into one, two, or four emitted records. The four-unit panels merge four operationally distinct supplied units into one, two, or four emitted records. The curves jointly vary grouping and the fixed per-group content budget; the content-fixed contrasts in Table~\ref{tab:phase9-main} provide the primary comparison. HUMAN contains 66 claims in 32 dependence components; WEB contains 138 claims in 135 components. Each domain and checkpoint is plotted separately.}
\label{fig:phase9-key-noise}
\end{figure}

The dose curves vary emitted group count together with the fixed 40-word budget per group and show system behavior as an estimated partition moves between one, two, and four groups.  Table~\ref{tab:phase9-main} reports both the one-group full-window inclusion gain and the content-fixed comparisons in which the ordered 160 attack words remain constant while record placement changes.

\subsection{Multiplicity-Adjusted Bootstrap Sensitivity}
\label{app:multiplicity-sensitivity}

\begin{table}[H]
\caption{Bonferroni familywise bootstrap sensitivity.  Counts show cells whose two-sided interval has a positive lower bound.  Adjusted 95\% intervals use percentile $\alpha/(2m)$ tails with $m=24$ for every Table~\ref{tab:phase9-main} row and $m=16$ for the matched six-slot row.}
\label{tab:multiplicity-sensitivity}
\centering
\small
\setlength{\tabcolsep}{5pt}
\begin{tabular}{@{}lrrr@{}}
\toprule
Reported subset & Cells & Ordinary 95\% & Familywise 95\% \\
\midrule
Table 1: full-content gain & 8 & 8/8 & 7/8 \\
Table 1: content-fixed split/merge & 16 & 16/16 & 16/16 \\
Table 1: all contrasts & 24 & 24/24 & 23/24 \\
Matched six-slot control & 16 & 12/16 & 11/16 \\
\bottomrule
\end{tabular}
\end{table}

We apply Bonferroni-adjusted percentile intervals within two design-distinct result families using 10,000 dependence-component bootstrap replicates regenerated from the frozen seeds.  For the 24 Table~\ref{tab:phase9-main} cells, simultaneous two-sided 95\% intervals retain positive lower bounds in 23 cells: all 16 content-fixed split/merge effects and seven of eight full-content gains.  The Phi-4-mini/WEB full-content gain is $+2.91$ points with adjusted interval $[-0.02,+5.95]$.  Directional familywise lower bounds are positive in all 24 cells.  For the 16 matched six-slot effects, 11 simultaneous two-sided intervals retain positive lower bounds.  These counts summarize multiplicity sensitivity; the model--domain estimates remain the reported effects.

\subsection{Observable Pipeline Stages}

\begin{table}[H]
\caption{Equal-domain strict accuracy (\%) for observable pipeline stages.  Each estimate averages mirrored directions and three presentations within item.}
\label{tab:stages}
\centering
\scriptsize
\begin{tabular}{@{}lrrrrrr@{}}
\toprule
Checkpoint & Edge & Count & Winner & Opaque 2:1 & Semantic 2:1 & End-to-end \\
\midrule
Qwen3-8B & 58.85 & 47.10 & 51.04 & 100.00 & 54.30 & 49.64 \\
Qwen3-4B & 54.26 & 48.93 & 48.94 & 98.93 & 53.26 & 49.64 \\
Phi-4-mini & 50.00 & 49.61 & 51.77 & 74.85 & 57.23 & 49.28 \\
\bottomrule
\end{tabular}
\end{table}

\subsection{Phase 6 Behavioral Gates}

Across the frozen Marker, M--P, and Rescue gates, all four checkpoints remain below threshold.  The provenance-equivalence gate is satisfied by Qwen3-8B, Qwen3-4B, and Mistral; Phi lies outside the equivalence band.  The Qwen3-8B tie leaves every gate unchanged under the half-tie sensitivity analysis.

\begin{table}[H]
\caption{Complete-candidate effects on the equal-domain 47-item panel.  Likelihood contrasts use 95\% item-bootstrap intervals.  Accuracy effects are percentage points; provenance uses a 90\% interval for the frozen equivalence test and the others use 95\% intervals.}
\label{tab:decision}
\centering
\scriptsize
\begin{tabular}{@{}lrrr@{}}
\toprule
Checkpoint & Marker LL & Prov. LL & M--P LL \\
\midrule
Qwen3-8B & $1.618[1.350,1.876]$ & $-0.349[-0.531,-0.170]$ & $1.967[1.590,2.321]$ \\
Qwen3-4B & $2.175[1.343,3.180]$ & $-0.243[-0.448,-0.045]$ & $2.418[1.553,3.435]$ \\
Phi-4-mini & $1.102[0.628,1.598]$ & $0.306[0.025,0.602]$ & $0.796[0.248,1.320]$ \\
Mistral-7B & $2.586[2.086,3.115]$ & $-0.213[-0.301,-0.127]$ & $2.799[2.279,3.347]$ \\
\bottomrule
\end{tabular}

\vspace{2pt}
\begin{tabular}{@{}lrrrr@{}}
\toprule
Checkpoint & Marker acc. & Prov. acc. & M--P acc. & Rescue acc. \\
\midrule
Qwen3-8B & $2.87[0.69,5.74]$ & $-0.72[-2.17,0.00]$ & $3.59[0.69,7.25]$ & $0.00[0.00,0.00]$ \\
Qwen3-4B & $2.90[0.00,6.52]$ & $0.00[-1.45,1.45]$ & $2.90[0.00,6.52]$ & $0.00[0.00,0.00]$ \\
Phi-4-mini & $3.47[-2.90,9.78]$ & $2.93[-1.36,7.31]$ & $0.54[-7.94,8.42]$ & $0.66[-3.62,5.62]$ \\
Mistral-7B & $3.59[0.72,6.52]$ & $0.00[0.00,0.00]$ & $3.59[0.72,6.49]$ & $-1.45[-3.62,0.00]$ \\
\bottomrule
\end{tabular}
\end{table}

\subsection{Phase 7 Primary Contrasts}

\begin{table}[H]
\caption{Same-root repetition contrasts in percentage points, equal-domain item means with 95\% bootstrap intervals.  A pass requires an estimate of at least 5 points and an interval lower bound above zero.}
\label{tab:phase7-contrasts}
\centering
\scriptsize
\begin{tabular}{@{}lrrrr@{}}
\toprule
Checkpoint & Raw D1--D8 & Prompt--Raw D8 & External--Raw D8 & External--Prompt D8 \\
\midrule
Qwen3-8B & $1.42[-0.03,3.20]$ & $0.35[-2.48,3.56]$ & $1.42[-0.03,3.19]$ & $1.07[-1.45,3.61]$ \\
Qwen3-4B & $4.35[1.09,8.70]$ & $0.00[-1.81,1.81]$ & $4.35[1.09,8.70]$ & $4.35[0.72,8.70]$ \\
Phi-4-mini & $27.02[21.60,32.43]$ & $-7.02[-12.94,-0.95]$ & $27.02[21.63,32.44]$ & $34.04[27.94,40.07]$ \\
\bottomrule
\end{tabular}
\end{table}

Phi satisfies the attack, external-recovery, and external-advantage gates.  The prompt-recovery pass count is 0/3.  Because External equals Raw $D=1$, external recovery and attack reversal are the same numerical contrast.

The Phi attack remains positive in both domains (30.43 and 23.61 points), all three presentations (38.41, 24.18, and 18.48 points), and both record orders (33.06 and 20.98 points).  Across its six presentation-by-order cells, the magnitude spans 4.35 to 48.64 points.  Qwen3-4B yields a 4.35-point aggregate change on \known{} items and five of 47 items.

\begin{figure}[H]
    \centering
    \includegraphics[width=0.66\linewidth]{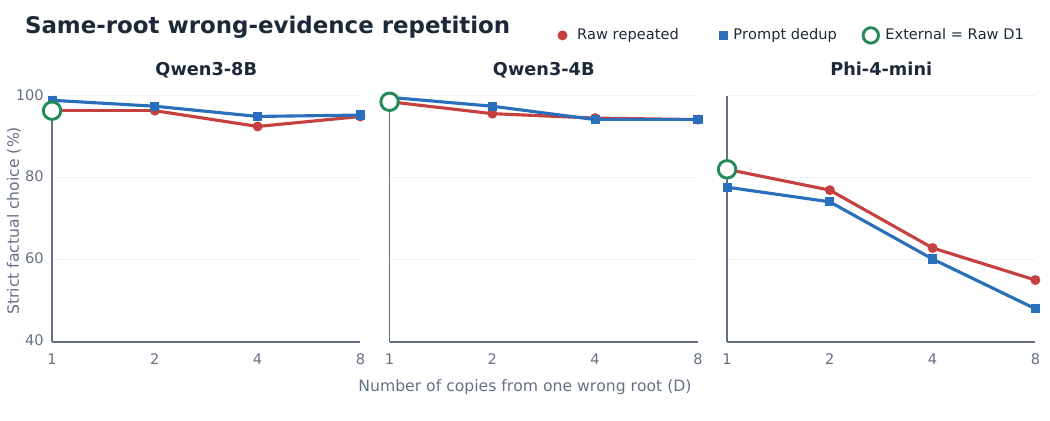}
    \caption{Strict factual choice in the controlled 47-item panel under repeated records from one wrong root.  Raw prompts hide shared roots; Prompt prompts show root IDs and a deduplication rule.  The green external point is the oracle-folded input and is exactly the Raw $D=1$ input.  The vertical axis starts at 40\%; points are equal-domain item means and paired-contrast intervals appear in Table~\ref{tab:phase7-contrasts}.}
    \label{fig:dose}
\end{figure}

\subsection{Phase 8 Gate Pattern and Token Sensitivity}

\begin{figure}[H]
    \centering
    \includegraphics[width=0.9\linewidth]{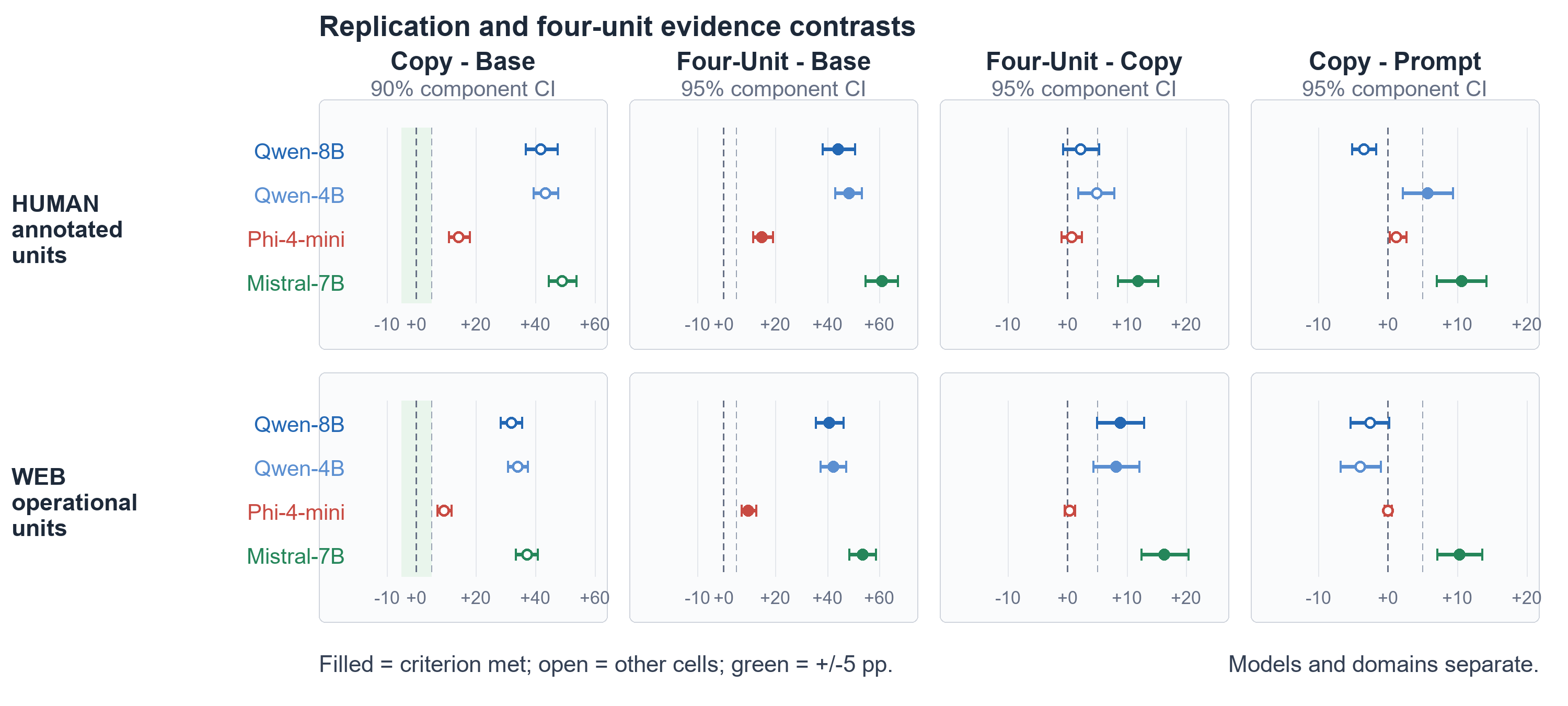}
    \caption{Record multiplicity changes complete-candidate decisions in both natural-evidence domains.  Copy--Base uses a 90\% interval and the shaded $\pm5$-point invariance band; other contrasts use 95\% intervals.}
    \label{fig:phase8}
\end{figure}

\begin{table}[t]
\caption{Phase 8 strict attack-side selection effects in percentage points.  HUMAN uses annotated evidence units.  WEB uses supplied canonical-URL units under the frozen operational separation criteria; a two-model consensus fixes stance labels.  Copy--Base uses a 90\% dependence-component bootstrap interval for the frozen $\pm5$-point invariance gate; the other columns use 95\% intervals.  Bold cells satisfy their prespecified gate.  Every domain and checkpoint is analyzed separately.}
\label{tab:phase8}
\centering
\scriptsize
\setlength{\tabcolsep}{3.2pt}
\begin{tabular}{@{}lrrrr@{}}
\toprule
Checkpoint & Copy--Base & Four-Unit--Base & Four-Unit--Copy & Copy--Prompt \\
\midrule
\multicolumn{5}{@{}l}{\textbf{HUMAN: annotated supplied evidence units}} \\
Qwen3-8B & $+41.83[+36.73,+47.30]$ & \textbf{$+44.06[+38.11,+50.57]$} & $+2.23[-0.74,+5.29]$ & $-3.47[-5.18,-1.72]$ \\
Qwen3-4B & $+43.32[+39.24,+47.58]$ & \textbf{$+48.27[+42.93,+53.04]$} & $+4.95[+1.74,+7.80]$ & \textbf{$+5.69[+2.13,+9.31]$} \\
Phi-4-mini & $+14.11[+10.82,+17.93]$ & \textbf{$+14.85[+11.25,+19.02]$} & $+0.74[-1.02,+2.39]$ & $+1.24[+0.23,+2.67]$ \\
Mistral-7B & $+49.01[+44.48,+53.68]$ & \textbf{$+60.89[+54.52,+67.07]$} & \textbf{$+11.88[+8.43,+15.23]$} & \textbf{$+10.64[+6.99,+14.18]$} \\
\addlinespace[2pt]
\multicolumn{5}{@{}l}{\textbf{WEB: operational supplied URL units; frozen stance consensus}} \\
Qwen3-8B & $+31.88[+28.26,+35.51]$ & \textbf{$+40.76[+35.43,+46.17]$} & \textbf{$+8.88[+4.96,+12.86]$} & $-2.54[-5.40,+0.18]$ \\
Qwen3-4B & $+34.06[+30.70,+37.41]$ & \textbf{$+42.21[+37.23,+47.08]$} & \textbf{$+8.15[+4.35,+12.04]$} & $-3.99[-6.87,-1.09]$ \\
Phi-4-mini & $+9.24[+6.93,+11.69]$ & \textbf{$+9.60[+6.83,+12.59]$} & $+0.36[-0.54,+1.27]$ & $+0.00[-0.54,+0.54]$ \\
Mistral-7B & $+37.14[+33.33,+40.78]$ & \textbf{$+53.44[+48.33,+58.45]$} & \textbf{$+16.30[+12.41,+20.29]$} & \textbf{$+10.33[+7.06,+13.59]$} \\
\bottomrule
\end{tabular}
\vspace{1pt}
\begin{minipage}{0.98\linewidth}\footnotesize
Positive values mean more influence from the manipulated side.  Certified upstream collapse reproduces BASE byte-for-byte; its recovery is the same paired contrast as Copy--Base.
\end{minipage}
\end{table}

Appendix Table~\ref{tab:phase8} contains all strict selection contrasts.  Across eight separately evaluated model--domain cells, replication invariance is 0/8, Four-Unit--Base sensitivity is 8/8, Four-Unit--Copy selectivity is 4/8, and prompt recovery is 3/8.  Model-specific effects and intervals remain the inferential unit; these counts summarize the eight frozen decisions.

\begin{table}[t]
\caption{Near-token-balanced sensitivity for Four-Unit--Copy.  The subset was frozen from tokenizer counts before model inference.  Intervals use the same within-domain dependence-component resampling rule.}
\label{tab:phase8-token-sensitivity}
\centering
\scriptsize
\setlength{\tabcolsep}{3.0pt}
\begin{tabular}{@{}llrr@{}}
\toprule
Domain & Checkpoint & $n$ & Selectivity [95\% CI] \\
\midrule
HUMAN & Qwen3-8B & 57 & $+0.44[-3.19,+4.17]$ \\
 & Qwen3-4B & 57 & $+6.14[+1.96,+9.69]$ \\
 & Phi-4-mini & 72 & $+0.69[-1.72,+2.87]$ \\
 & Mistral-7B & 50 & $+11.00[+5.77,+15.69]$ \\
\addlinespace[2pt]
WEB & Qwen3-8B & 79 & $+9.49[+4.17,+14.87]$ \\
 & Qwen3-4B & 79 & $+8.86[+3.12,+14.74]$ \\
 & Phi-4-mini & 89 & $-0.28[-1.37,+0.57]$ \\
 & Mistral-7B & 57 & $+14.47[+7.89,+20.61]$ \\
\bottomrule
\end{tabular}
\end{table}

The near-token-balanced sensitivity subset uses the pre-inference criterion of a mean absolute Copy--Four-Unit input-token difference bounded by 20 over the four attack-side/order cells.  The prespecified gates apply to the full analysis.  Full strict, half-tie, attack-probability, likelihood-margin, item-bootstrap, and component-bootstrap outputs remain in the four analysis files; the independent audit recomputes the primary contrasts through a separate implementation.

\subsection{Threshold Sensitivity}

\begin{table}[H]
\centering
\caption{Threshold sensitivity. Cells show passes/eligible model--domain cells. Equivalence uses a 90\% CI within $[-\delta,+\delta]$; positive effects require an estimate of at least $\delta$ and a 95\% lower bound above zero. The sweep reports the frozen gates across four operational margins.}
\label{tab:threshold-sensitivity}
\small
\begin{tabular}{llrrrr}
\toprule
Phase & Contrast & 2.5 pp & 5 pp & 7.5 pp & 10 pp \\
\midrule
Phase6 & Marker positive effect & 2/4 & 0/4 & 0/4 & 0/4 \\
Phase6 & Provenance equivalence & 3/4 & 3/4 & 4/4 & 4/4 \\
Phase6 & Marker--provenance positive effect & 2/4 & 0/4 & 0/4 & 0/4 \\
Phase6 & Topology-rescue positive effect & 0/4 & 0/4 & 0/4 & 0/4 \\
Phase8 & Copy--Base equivalence & 0/8 & 0/8 & 0/8 & 0/8 \\
Phase8 & Four-Unit--Base positive effect & 8/8 & 8/8 & 8/8 & 7/8 \\
Phase8 & Four-Unit--Copy positive effect & 5/8 & 4/8 & 4/8 & 2/8 \\
Phase8 & Copy--Prompt positive effect & 3/8 & 3/8 & 2/8 & 2/8 \\
\bottomrule
\end{tabular}
\end{table}

Copy--Base lies outside equivalence in all eight cells at every tested margin, so the invariance result is stable across the 2.5--10-point sweep.  Four-Unit--Base remains positive and material in eight of eight cells through 7.5 points and seven of eight at 10 points.  Selectivity changes from 5/8 at 2.5 points to 2/8 at 10 points, while prompt recovery changes from 3/8 to 2/8.  Model-specific effects and intervals are primary; gate counts provide an operational digest.

\subsection{Controlled-Generation Diagnostics}

\begin{table}[H]
\caption{Diagnostics for controlled answer-plus-one-sentence generation.  Each model--domain row contains 480 generated trials: 40 items by three conditions, two attack sides, and two orders.  Invalid is computed over those 480 trials.  Agreement uses the same matched trials; all are generation-valid and candidate-untied.  The parsed leading answer is the scored endpoint.}
\label{tab:phase10-diagnostics}
\centering
\scriptsize
\begin{tabular}{@{}llrr@{}}
\toprule
Checkpoint & Domain & Invalid (\%) & Choice agreement (\%) \\
\midrule
Qwen3-8B & HUMAN & 0.00 & 65.62 \\
Qwen3-8B & WEB & 0.00 & 70.83 \\
Qwen3-14B-AWQ & HUMAN & 0.00 & 63.75 \\
Qwen3-14B-AWQ & WEB & 0.00 & 74.17 \\
Phi-4 & HUMAN & 0.00 & 65.62 \\
Phi-4 & WEB & 0.00 & 77.08 \\
\bottomrule
\end{tabular}
\end{table}

Every formal output has a valid leading answer, making all-trial and valid-output estimates identical.  Agreement quantifies consistency between the candidate and generated interfaces, which expose different decision rules while their Copy--Base effects share direction.

\section{Run Lineage and Reproducibility}
\label{app:repro}

Inference precision is frozen per checkpoint: the original runs use bfloat16, Qwen3-14B uses its official AWQ checkpoint, Phi-4 uses dynamic NF4, and the Qwen3-8B quantization control uses dynamic NF4.  Candidate scoring is deterministic, generation is greedy, and each experiment freezes its bootstrap seed before treatment inference.  Model-specific chat templates and native assistant termination sequences are fixed during compilation.

Phase 5 completes 1,692/1,692 trials per checkpoint for Qwen3-8B, Qwen3-4B, and Phi.  Phase 6 completes 1,692/1,692 per checkpoint for those three and Mistral.  Phase 7 completes 2,538/2,538 per checkpoint for the two Qwen models and Phi.  Phase 8 completes 3,824/3,824 per checkpoint for all four original models.  The extended candidate grid adds 3,824/3,824 trials in each of three inference configurations; controlled generation adds 960/960 in each of three; grouping stress adds 9,792/9,792 in each of four; matched serialization adds 3,264/3,264 in each of four; the controlled GEO panel adds 768/768 in each of four.  These totals sum to \FinalPaperTrialsRevisionThree{} formal evaluation trials across six unique checkpoint identities.  Quantized reruns contribute executed trials under their existing checkpoint identities.  Phase 11's \PhaseElevenChoiceTrials{} choices require \PhaseElevenCandidateForwardCalls{} physical candidate forward calls, and Phase 12's \PhaseTwelveChoiceTrials{} choices require \PhaseTwelveCandidateForwardCalls{} calls.  This implementation accounting is reported separately from the paper's decision-trial total.  The anonymous ZIP's per-trial output layer spans Phase 6 and Phase 8--12; Phase 5 and Phase 7 are represented by independent internal result audits.

Independent scripts recheck prompt reconstruction, candidate likelihood arithmetic or leading-answer parsing, typed labels, item aggregation, dependence components, and every reported bootstrap statistic.  The grouping, added-checkpoint, controlled-generation, matched-serialization, and controlled-GEO audits all have pass status and zero errors.  The Phase 12 audit verifies four 768-trial model runs, exact key coverage, and item-first 10,000-replicate bootstrap estimates recomputed from raw predictions.  Audit recomputation independently verifies stored outputs and every derived statistic; model inference is represented by the executed forward trials counted above.

The anonymous result-recomputation artifacts use portable paths and identity-scrubbed metadata.  The Phase 9--12 tables, auxiliary curve, generated macros, and JSON snapshots are produced after their independent empirical audits pass status, error-count, and trial-accounting checks.

The anonymous result-recomputation artifact retains the verified Phase 6 and Phase 8--11 layers and adds the complete Phase 12 result-verification extension.  The ZIP is 64,750,897 bytes; its 302 entries sit under one clean root and contain 285 files.  The top manifest covers all 284 non-manifest files.  In a fresh extraction, the combined verifier and standard-library recomputation scripts return \texttt{PASS} with zero failures, including all 3,072 Phase 12 predictions, grouping assignments, five rebuilt paper outputs, and an independently recomputed maximum numerical difference of $5.33\times10^{-15}$.  Recursive scans across the 285 text and manifest files report zero host-path, identity, device-identifier, model-weight, cache, bytecode, or nested-ZIP findings.  The artifact supports stored-output recomputation; fresh forward execution uses public checkpoints and licensed source texts.

\end{document}